\def\year{2019}\relax
\documentclass[letterpaper]{article} 
\usepackage{aaai19}  
\usepackage{times}  
\usepackage{helvet} 
\usepackage{courier}  
\usepackage[hyphens]{url}  
\usepackage{graphicx} 
\usepackage{subfigure}
\usepackage{booktabs}
\usepackage{multirow}
\usepackage{amssymb}
\usepackage{amsmath}
\usepackage{amsthm}
\usepackage{algorithm}
\usepackage{color}
\usepackage{algorithmic}
\def\UrlFont{\rm}  
\usepackage{graphicx}  
\newtheorem{assumption}{Assumption}
\newtheorem{theorem}{Theorem}
\newtheorem{lemma}{Lemma}
\newtheorem{proposition}{Proposition}

\title{Gradient Q$(\sigma, \lambda)$: A Unified Algorithm with Function Approximation\\ for Reinforcement Learning }
 
\begin{document}
\author{Long Yang, Yu Zhang, Qian Zheng, Pengfei Li, Gang Pan
 \\
  Department of Computer Science,
  Zhejiang University\\
  \texttt{\{yanglong,hzzhangyu,qianzheng,pfl,gpan\}@zju.edu.cn} 
}

\maketitle

\begin{abstract}
Full-sampling (e.g., Q-learning) and pure-expectation (e.g., Expected Sarsa) algorithms are efficient and frequently used techniques in reinforcement learning.
Q$(\sigma,\lambda)$ is the first approach unifies them with eligibility trace through the sampling degree $\sigma$.
However, it is limited to the tabular case, for large-scale learning, the Q$(\sigma,\lambda)$ is too expensive to require a huge volume of tables to accurately storage value functions.
To address above problem,  we propose a GQ$(\sigma,\lambda)$
that extends tabular Q$(\sigma,\lambda)$ with linear function approximation.
We prove the convergence of GQ$(\sigma,\lambda)$.
Empirical results on some standard domains show that GQ$(\sigma,\lambda)$ with a combination of full-sampling with pure-expectation reach a better performance than full-sampling and pure-expectation methods.
\end{abstract}

\section{Introduction}

Reinforcement learning (RL) is a powerful tool for sequential decision-making problem. 
In RL, the agent's goal is to learn from experiences and to seek an optimal policy from the delayed reward decision system. 
Tabular learning methods are core ideas in RL algorithms with simple forms: \emph{value functions} are represented as arrays, or tables \cite{sutton1998reinforcement}. 
One merit of tabular learning methods is that such methods converge to the optimal solution with solid theoretical guarantee \cite{singh2000convergence}.
However, when state space is enormous, we suffer from what Bellman called ``curse of dimensionality" \cite{ernest1957dynamic}, and can not expect to obtain value function accurately by tabular learning methods. 
An efficient approach to address the above problem is to use a parameterized function to approximate the value function \cite{sutton1998reinforcement}.

Recently, Yang et al.~\shortcite{yang2018} propose a new algorithm Q$(\sigma,\lambda)$ that develops Q$(\sigma)$ \cite{sutton2018reinforcement,de2018multi} with eligibility trace.
Q$(\sigma,\lambda)$ unifies Sarsa$(\lambda)$ \cite{rummery1994line} and Q$^{\pi}(\lambda)$ \cite{H2016}.
However, original theoretical results by Yang et al. \shortcite{yang2018} are limited in tabular learning.
In this paper, we extend tabular Q$(\sigma,\lambda)$ algorithm with linear function approximation and propose the gradient Q$(\sigma,\lambda)$ (GQ$(\sigma,\lambda)$) algorithm.
The proposed GQ$(\sigma,\lambda)$ algorithm unifies \emph{full-sampling} ($\sigma=1$) and \emph{pure-expectation} ($\sigma=0$) algorithm through the \emph{sampling degree} $\sigma$.
Results show that GQ$(\sigma,\lambda)$ with a varying combination between full-sampling with pure-expectation achieves a better performance than both full-sampling and pure-expectation methods.

Unfortunately, it is not sound to expend Q$(\sigma,\lambda)$ by semi-gradient method via \emph{mean square value error} (MSVE) objective function directly,
although the linear, semi-gradient method is arguably the simplest and best-understood kind of function approximation. 
In this paper, we provide a profound analysis of the instability of Q$(\sigma,\lambda)$ with function approximation by the semi-gradient method.

Furthermore, to address above instability, we propose GQ$(\sigma,\lambda)$ algorithm under the framework of \emph{mean squared projected Bellman error} (MSPBE) \cite{sutton2009fast_a}.
However, as pointed out by Liu et al.~\shortcite{liu2015finite}, we can not get an unbiased estimate of the gradient with respect to the MSPBE objective function. 
In fact, since the update law of gradient involves the product of expectations, 
the unbiased estimation cannot be obtained via a single sample, which is the double-sampling problem. 
Secondly, the gradient of MSPBE objective function has a term likes $\mathbb{E}[\phi_{t} \phi_{t}^\top]^{-1}$, cannot also be estimated via a single sample, which is the second bottleneck of applying stochastic gradient method to optimize MSPBE objective function. 
Inspired by the key step of the derivation of TDC algorithm \cite{sutton2009fast_a}, 
we apply the two-timescale stochastic approximation \cite{borkar2000ode} to address the dilemma of double-sampling problem, 
and propose a convergent GQ$(\sigma,\lambda)$ algorithm which unifies full-sampling and pure-expectation algorithm with function approximation.

Finally, we conduct extensive experiments on some standard domains to show that GQ$(\sigma,\lambda)$ with an value $\sigma\in(0,1)$ that results in a mixture of the full-sampling with pure-expectation methods, performs better than either extreme $\sigma=0$ or $\sigma=1$.

\section{Background and Notations}

The standard reinforcement learning framework~\cite{sutton1998reinforcement} is often formalized as \emph{Markov decision processes} (MDP)~\cite{puterman2014markov}. 
It considers 5-tuples form $\mathcal{M}=(\mathcal{S},\mathcal{A},\mathcal{P},\mathcal{R},\gamma)$, 
where $\mathcal{S}$ indicates the set of all states, $\mathcal{A}$ indicates the set of all actions.
At each time $t$, the agent in a state $S_{t}\in\mathcal{S}$ and it takes an action $A_{t}\in\mathcal{A}$, 
then environment produces a reward $R_{t+1}$ to the agent.
$\mathcal{P} : \mathcal{S}\times\mathcal{A}\times\mathcal{S}\rightarrow[0,1]$,
$P_{s s^{'}}^a=\mathcal{P}(S_{t}=s^{'}|S_{t-1}=s,A_{t-1}=a)$ is the conditional probability for the state transitioning from $s$ to $s^{'}$ under taking the action $a$.
$\mathcal{R} : \mathcal{S}\times\mathcal{A}\rightarrow\mathbb{R}^{1}$: $\mathcal{R}_{s}^{a}=\mathbb{E}[R_{t+1}|S_{t}=a,A_{t}=a]$. 
The discounted factor $\gamma\in(0,1)$.

A \emph{policy} is a probability distribution defined on $\mathcal{S}\times\mathcal{A}$, 
\emph{target policy} is the policy that will be learned,
and \emph{behavior policy} is used to generate behavior.
If $\mu=\pi$, the algorithm is called \emph{on-policy} learning, otherwise it is \emph{off-policy} learning.
We assume that Markov chain induced by behavior policy $\mu$ is ergodic, 
then there exists a stationary distribution $\xi$ such that 
$\forall S_{0}\in\mathcal{S}$
\[\frac{1}{n}\sum_{k=1}^{n} P(S_{k}= s |S_{0})\rightarrow \xi(s),~\text{as}~ n\rightarrow\infty.\] 
We denote $\Xi = \text{diag}\{\xi(s_{1}),\xi(s_{2}),\cdots,\xi(s_{|\mathcal{S}|)}\}$ 
as a diagonal matrix and its diagonal element is the stationary distribution of state.

For a given policy $\pi$, one of many key steps in RL is to estimate the \emph{state-action value function} 
\[
q^{\pi}(s,a) = \mathbb{E}_{\pi}[G_{t}|S_{t} = s,A_{t}=a],
\]
where $G_{t}=\sum_{k=0}^{\infty}\gamma^{k}R_{k+t+1}$ and $\mathbb{E}_{\pi}[\cdot|\cdot]$ stands for the expectation of a random variable with respect to the probability distribution induced by $\pi$.
It is known that $q^{\pi}(s,a)$ is the unique fixed point of \emph{Bellman operator} $\mathcal{B}^{\pi}$,
\begin{flalign}
\label{bellman-equation}
\mathcal{B}^{\pi} q^{\pi}=q^{\pi}~~\text{(Bellman equation)},
\end{flalign}
where Bellman operator $\mathcal{B}^{\pi}$ is defined as:
\begin{flalign}
\label{Eq:bellman_operator}
\mathcal{B}^{\pi} q&=\mathcal{R}^{\pi}+\gamma {P}^{\pi}q,
\end{flalign}
 $P^{\pi}$$\in\mathbb{R}^{|\mathcal{S}| \times |\mathcal{S}|}$ and $R$$\in\mathbb{R}^{|\mathcal{S}|\times|\mathcal{A}|}$ with the corresponding elements:
$
 P^{\pi}_{ss^{'}}= \sum_{a \in \mathcal{A}}\pi(a|s)P^{a}_{ss^{'}},R(s,a)=\mathcal{R}_{s}^{a}.
$

\subsection{Temporal Difference Learning and $\lambda$-Return}

We can not calculate $q^{\pi}$ from Bellman equation (\ref{bellman-equation}) directly for the model-free RL problem (in such problem, the agent can not get $\mathcal{P}$ or $\mathcal{R}$ for a given MDP). 
In RL, temporal difference (TD) learning \cite{sutton1988learning} is one of the most important methods to solve the model-free RL problem.
$\lambda$-Return is a multi-step TD learning that needs a longer sequence of experienced rewards is to learning the value function. 

\textbf{TD learning} One-step TD learning estimates the value function by taking action according to behavior policy, sampling the reward, and bootstrapping via the current estimation of the value function. 
Sarsa~\cite{rummery1994line}, Q-learning~\cite{watkins1989learning} and Expected-Sarsa~\cite{van2009theoretical} are typical one-step TD learning algorithms.

From the view of sampling degree, TD learning methods fall into two categories: \emph{full-sampling} and \emph{pure-expectation}, which is deeply discussed in section 7.5\&7.6 in \cite{sutton2018reinforcement} or \cite{de2018multi}.

Sarsa and Q-learning are typical full-sampling algorithms which have sampled all transitions to learn value function.
Instead,
pure-expectation algorithms take into account how likely each action is 
under current policy,
e.g. \emph{Tree-Backup}~\cite{precup2000eligibility} or Expected-Sarsa uses the expectation of state-action value to estimate value function.
$\text{Q}^{\pi}(\lambda)$~\cite{H2016} algorithm is also a pure-expectation algorithm which combines TD learning with eligibility trace.
Harutyunyan et al.\shortcite{H2016} prove that when behavior and target policies are sufficiently close, off-policy $\text{Q}^{\pi}(\lambda)$ algorithm converges both in policy evaluation and control task.

\textbf{$\lambda$-Return} 
For a trajectory, the $\lambda$-return is an average contains all the $n$-step returns by weighting proportionally to $\lambda^{n-1}$, $\lambda\in[0,1]$. 
Since the mechanisms of all the $\lambda$-return algorithms are similar, we only present the definition of $\lambda$-return of Sarsa$(\lambda)$~\cite{sutton2018reinforcement} as follows,
\begin{flalign}
\label{l_return_sarsa_onpolicy}
G_{t}^{\lambda,\text{S}}=(1-\lambda)\sum_{n=0}^{\infty}\lambda^{n}G_{t}^{t+n},
\end{flalign}
where $G_{t}^{t+n}=\sum_{i=0}^{n}\gamma^{i}R_{t+i+1}+\gamma^{n+1}Q(S_{t+n},A_{t+n})$ is $n$-step return of Sarsa from time $t$.
After some simple algebra, the $\lambda$-return can be rewritten as a sum of TD-error,
\begin{flalign}
\label{lam_oprator}
G_{t}^{\lambda,\text{S}}
=Q(S_0,A_0)+\sum_{t=0}^{\infty}(\lambda\gamma)^{t}\delta_{t}^{\text{S}},
\end{flalign}
where $\delta_{t}^{\text{S}}=R_{t+1}+\gamma Q_{t+1} - Q_{t}$, and $Q_{t}\overset{\text{def}}=Q(S_t,A_t)$.

\subsection{An Unified View}

In this section, we introduce the existing method that unifies full-sampling and pure-expectation algorithms.

\textbf{$\text{Q}(\sigma)$ Algorithm}
Recently, Sutton and Barto\shortcite{sutton2018reinforcement} and De Asis et al.\shortcite{de2018multi} propose a new TD algorithm $\text{Q}(\sigma)$ unifies Sarsa and Expected Sarsa 
\footnote{For multi-step case, $\text{Q}(\sigma)$ unifies \emph{$n$-step Sarsa} and \emph{$n$-step Tree-Backup}~\cite{precup2000eligibility}. For more details, please refer to \cite{de2018multi}.}. 
$\text{Q}(\sigma)$ estimates value function by weighting the average between Sarsa and Expected-Sarsa through a \emph{sampling parameter} $\sigma$. 
For a transition ($S_{t},A_{t},R_{t+1},S_{t+1},A_{t+1}$),
$\text{Q}(\sigma)$ updates value function as follows, 
\begin{flalign}
\nonumber
Q(S_{t},A_{t})&=Q(S_{t},A_{t}) + \alpha_{t}\delta_{t}^{\sigma},\\
\label{eq:delta_sigma}
\delta_{t}^{\pi,\sigma}&=\sigma\delta_{t}^{\text{S}}+(1-\sigma)\delta_{t}^{\text{ES}},
\end{flalign}
$\delta_{t}^{\text{ES}}=R_{t+1}+\mathbb{E}_{\pi}[Q(S_{t+1},\cdot)]- Q_{t}$,
$\mathbb{E}_{\pi}[Q(S_{t+1},\cdot)]=\sum_{a\in\mathcal{A}}\pi(a|S_{t+1})Q(S_{t+1},a)$.

Q$(\sigma)|_{\sigma=0}$ is reduced to Expected-Sarsa, while Q$(\sigma)|_{\sigma=1}$ is Sarsa exactly.
Experiments by De Asis et al.\shortcite{de2018multi} show that for an intermediate value of $\sigma\in(0,1)$, which results in a mixture of the existing algorithms, performs better than either extreme $\sigma=0$ or $\sigma=1$.

\textbf{Q($\sigma,\lambda$) Algorithm} Later, Yang et al.\shortcite{yang2018} extend Q($\sigma$) with eligibility trace, and they propose Q($\sigma,\lambda$)
unifies Sarsa($\lambda$) and $\text{Q}^{\pi}(\lambda)$.
$\text{Q}(\sigma,\lambda)$ updates value function as:
\begin{flalign}
    e(s,a)&=\gamma\lambda e(s,a)+\mathbb{I}\{(S_{t},A_{t})=(s,a)\},\\
    Q_{t+1}(s,a)&= Q_{t+1}(s,a) + \alpha_{t}\delta_{t}^{\pi,\sigma}e(s,a),
\end{flalign}
where $\mathbb{I}$ is indicator function, $\alpha_{t}$ is step-size.

We notice that $\text{Q}(\sigma,\lambda)|_{\sigma=0}$ is reduced to $\text{Q}^{\pi}(\lambda)$~\cite{H2016}, Q$(\sigma,\lambda)|_{\sigma=1}$ is Sarsa($\lambda$) exactly.
The experiments in \cite{yang2018} shows a similar conclusion as De Asis et al.\shortcite{de2018multi}:
an intermediate value of $\sigma\in(0,1)$ achieve the best performance than extreme $\sigma=0$ or $\sigma=1$.
Besides, Yang et al.\shortcite{yang2018} have showed that for a trajectory $\{(S_{t},A_{t},R_{t})\}_{t\ge0}$, by Q($\sigma,\lambda$), the total update of a given episode reaches
\begin{flalign}
\label{mixedoperator}
Q(S_{0},A_{0})+\sum_{t=0}^{\infty}(\lambda\gamma)^{t}\delta^{\pi,\sigma}_{t},
\end{flalign}
which is an off-line version of Q($\sigma$) with eligibility trace. If $\sigma=1$, Eq.(\ref{mixedoperator}) is Eq.(\ref{lam_oprator}) exactly.  

Finally, we introduce the \emph{mixed-sampling operator} $\mathcal{B}^{\pi,\mu}_{\sigma,\lambda}$~\cite{yang2018}, which is a high level view of (\ref{mixedoperator}),
\begin{flalign}
\nonumber
\mathcal{B}^{\pi,\mu}_{\sigma,\lambda}: \mathbb{R}^{|\mathcal{S}|\times|\mathcal{A}|}&\rightarrow\mathbb{R}^{|\mathcal{S}|\times|\mathcal{A}|}\\
\label{def:mixedoperator}
q&\mapsto q+\mathbb{E}_{\mu}[\sum_{t=0}^{\infty}(\lambda\gamma)^{t}\delta^{\pi,\sigma}_{t}].
\end{flalign}

\section{Q$(\sigma,\lambda)$ with Semi-Gradient Method}

In this section, we analyze the instability of extending tabular Q$(\sigma,\lambda)$ with linear function approximation by the semi-gradient method.
We need some necessary notations about linear function approximation.

When the dimension of $\mathcal{S}$ is huge, we cannot expect to obtain value function accurately by tabular learning methods.  
We often use a linear function with a parameter ${\theta}$ to estimate $q^{\pi}(s,a)$ as follows,
\[
q^{\pi}(s,a)\approx\phi^{\top}(s,a)\theta\overset{\text{def}}=\hat{Q}_{\theta}(s,a),
\]
where $\phi: \mathcal{S}\times\mathcal{A}\rightarrow\mathbb{R}^{p}$  is a \emph{feature map}, specifically,
$\phi(s,a)=(\varphi_{1}(s,a),\varphi_{2}(s,a),\cdots,\varphi_{p}(s,a))^{\top},$
the corresponding element $\varphi_{i}$ is defined as follows $\varphi_{i}:\mathcal{S}\times\mathcal{A}\rightarrow\mathbb{R}$.
Then $\hat{Q}_{\theta}$ can be written as a matrix version, \[\hat{Q}_{\theta}=\Phi\theta\approx q^{\pi},\]
where $\Phi$ is a $|\mathcal{S}||\mathcal{A}|\times p$ matrix whose rows are the state-action features $\phi(s,a)$,$(s,a)\in\mathcal{S}\times\mathcal{A}$.

\subsection{Semi-Gradient Method}
For a trajectory $\{(S_{k},A_{k},R_{k})\}_{k\ge0}$, 
we define the update rule of Q($\sigma,\lambda$) with semi-gradient method as follows,
\begin{flalign}
\nonumber
\theta_{k+1}
&=\theta_{k}-\alpha_{k}\nabla_{\theta}\big(G_{k}^{\lambda}-\hat{Q}_{\theta}(S_{k},A_{k})\big)^2|_{\theta=\theta_k}\\
\nonumber
&=\theta_{k}-\alpha_{k}\{G_{k}^{\lambda}-\hat{Q}_{\theta_{k}}(S_{k},A_{k})\}\nabla_{\theta}(-\hat{Q}_{\theta}(S_{k},A_{k}))|_{\theta=\theta_k}\\
\label{Eq:semi-gradient}
&=\theta_{k}+\alpha_{k}\{\sum_{t=k}^{\infty}(\lambda\gamma)^{t-k}\delta^{\pi,\sigma}_{k,t}(\theta_{k})\}\phi(S_{k},A_{k}), 
\end{flalign}
where $\alpha_{k}$ is step-size, $G_{k}^{\lambda}$ is an off-line estimate of value function according to Eq.(\ref{mixedoperator}), specifically, for each $k\ge1$,
\begin{flalign}
\label{G}
G_{k}^{\lambda}&=\theta_{k}^{\top}\phi_{k}+\sum_{t=k}^{\infty}(\lambda\gamma)^{t-k}\delta^{\pi,\sigma}_{k,t}(\theta_{k}),\\
\nonumber
\delta^{\pi,\sigma}_{k,t}(\theta_{k})&=\sigma\delta_{t}^{\text{S}}(\theta_{k})+(1-\sigma)\delta_{t}^{\text{ES}}(\theta_{k}),
\end{flalign}
where 
$
\delta_{t}^{\text{S}}(\theta_{k})=R_{t}+\gamma \theta_{k}^{\top}\phi_{t+1}-\theta_{k}^{\top}\phi_{t},~~
\delta_{t}^{\text{ES}}(\theta_{k})= R_{t}+\gamma\mathbb{E}_{\pi}[\theta_{k}^{\top}\phi(S_{t+1},\cdot)]-\theta_{k}^{\top}\phi_{t},
$
and $\phi_{t}\overset{\text{def}}=\phi(S_{t},A_{t})$.

\subsection{Instability Analysis}
Now, we show the iteration (\ref{Eq:semi-gradient}) is an unstable algorithm.
Let's consider the sequence $\{\theta_k\}_{k\ge0}$ generated by the iteration (\ref{Eq:semi-gradient}), then the following holds,
\begin{flalign}
\label{Eq:linear_eq}
\mathbb{E}[\theta_{k+1}|\theta_{0}]=\mathbb{E}[\theta_{k}|\theta_{0}]+\alpha_{k}(A_{\sigma}\hspace{0.05cm}\mathbb{E}[\theta_{k}|\theta_{0}]+b_{\sigma}),
\end{flalign}
where $A_{\sigma}=\mathbb{E}[\hat{A}_k]$ and $b_{\sigma}=\mathbb{E}[\hat{b}_k]$, where
\begin{flalign}
\nonumber
\hat{A}_{k}&=\phi_{k}\sum_{t=k}^{\infty}(\lambda\gamma)^{t-k}\big(\sigma(\gamma\phi_{t+1}-\phi_{t})+\\
\label{A_k}
&~~~~~~~~~~~~~~~~~~~~(1-\sigma)[\gamma\mathbb{E}_{\pi}\phi(S_{t+1},\cdot)-\phi_{t}]\big)^{\top},\\
\label{b_k}
\hat{b}_{k}&=\sum_{t=k}^{\infty}(\lambda\gamma)^{t-k}R_{t}\phi_{t}.
\end{flalign}
Furthermore,
\begin{flalign}
A_{\sigma}
\nonumber
&={\Phi}^{\top}\Xi(I-\gamma\lambda{{P}}^{\mu})^{-1}((1-\sigma)\gamma{P}^{\pi}+\sigma\gamma{P}^{\mu}-{I}){\Phi},\\
b_{\sigma}
&=\Phi\Xi (I-\gamma\lambda{{P}}^{\mu})^{-1}(\sigma\mathcal{R}^{\mu}+(1-\sigma)\mathcal{R}^{\pi}),
\end{flalign}
Eq.(\ref{Eq:linear_eq}) plays a critical role for our analysis,  we provide its proof in Appendix A.

As the same discussion by Tsitsiklis and Van Roy~\shortcite{tsitsiklis1997analysis}; Sutton, Mahmood, and White~\shortcite{sutton2016emphatic}, 
under the conditions of Proposition 4.8 proved by Bertsekas and Tsitsiklis \shortcite{bertsekas1995neuro}, 
\emph{if $A_{\sigma}$ is a negative matrix, then
$\theta_{k}$ generated by iteration (\ref{Eq:semi-gradient}) is a convergent algorithm. By (\ref{Eq:linear_eq}), $\theta_{k}$ converges to the unique TD fixed point $\theta^{*}$}:
\begin{flalign}
\label{TD-fixed-point}
A_{\sigma}\theta^{*}+b_{\sigma}=0.
\end{flalign}

For on-policy case, for $\forall \sigma\in[0,1]$, 
\begin{flalign}
\label{Eq:on_policy_key_matrix}
A_{\sigma}={\Phi}^{\top}\Xi(I-\gamma\lambda{{P}}^{\pi})^{-1}(\gamma{P}^{\pi}-{I}){\Phi}.
\end{flalign}
It has been shown that $A_{\sigma}$ in Eq.(\ref{Eq:on_policy_key_matrix}) is negative definite (e.g. section 9.4 in \cite{sutton2018reinforcement}), thus iteration~(\ref{Eq:semi-gradient}) is a convergent algorithm: it converges to $\theta^{*}$ satisfies (\ref{TD-fixed-point}).

Unfortunately, by the fact that the steady state-action distribution doesn't match the transition probability during off-policy learning, $\forall \sigma\in(0,1)$,  $A_{\sigma}$ may
not have an analog of~(\ref{Eq:on_policy_key_matrix}).
Thus, unlike on-policy learning, there is no guarantee that  $A_{\sigma}$ keeps the negative definite property, thus $\theta_{k}$ may diverge. 
We use a typical example to illustrate it.

\subsection{A Counter Unstable Example}

\begin{figure}[h]
    \centering
    \includegraphics[scale=0.7, trim={1mm 1mm 1mm 1mm}, clip]{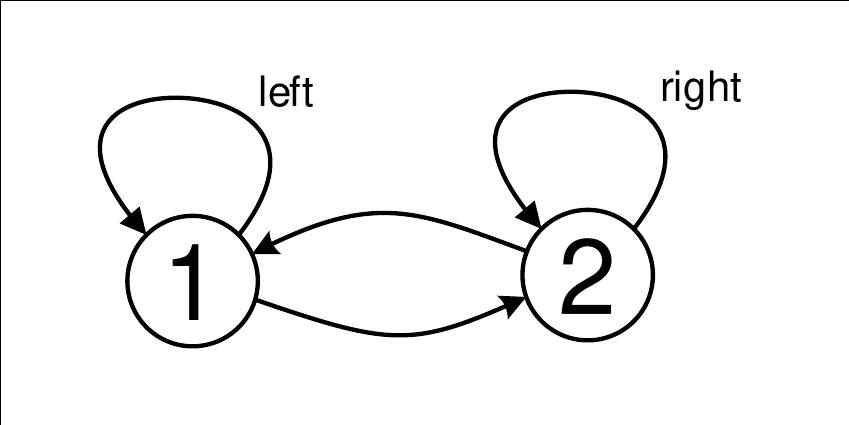}
    \caption{A Counter Example \cite{touati2018convergent}. We assign the features $ \{(1, 0)^{\top}, (2, 0)^{\top}, (0, 1)^{\top}, (0, 2)^{\top}\}$ to the state-action pairs $\{(1,\mathtt{right}),(2,\mathtt{right}),(1,\mathtt{left}),(2,\mathtt{left})\}$, the target policy $\pi(\mathtt{right} |\cdot)=1$ and the behavior policy $\mu(\mathtt{right} | \cdot)=0.5$.
    }
\end{figure}

\begin{figure}[h]
\label{theta_div}
    \centering
    \includegraphics[width=0.35\textwidth]{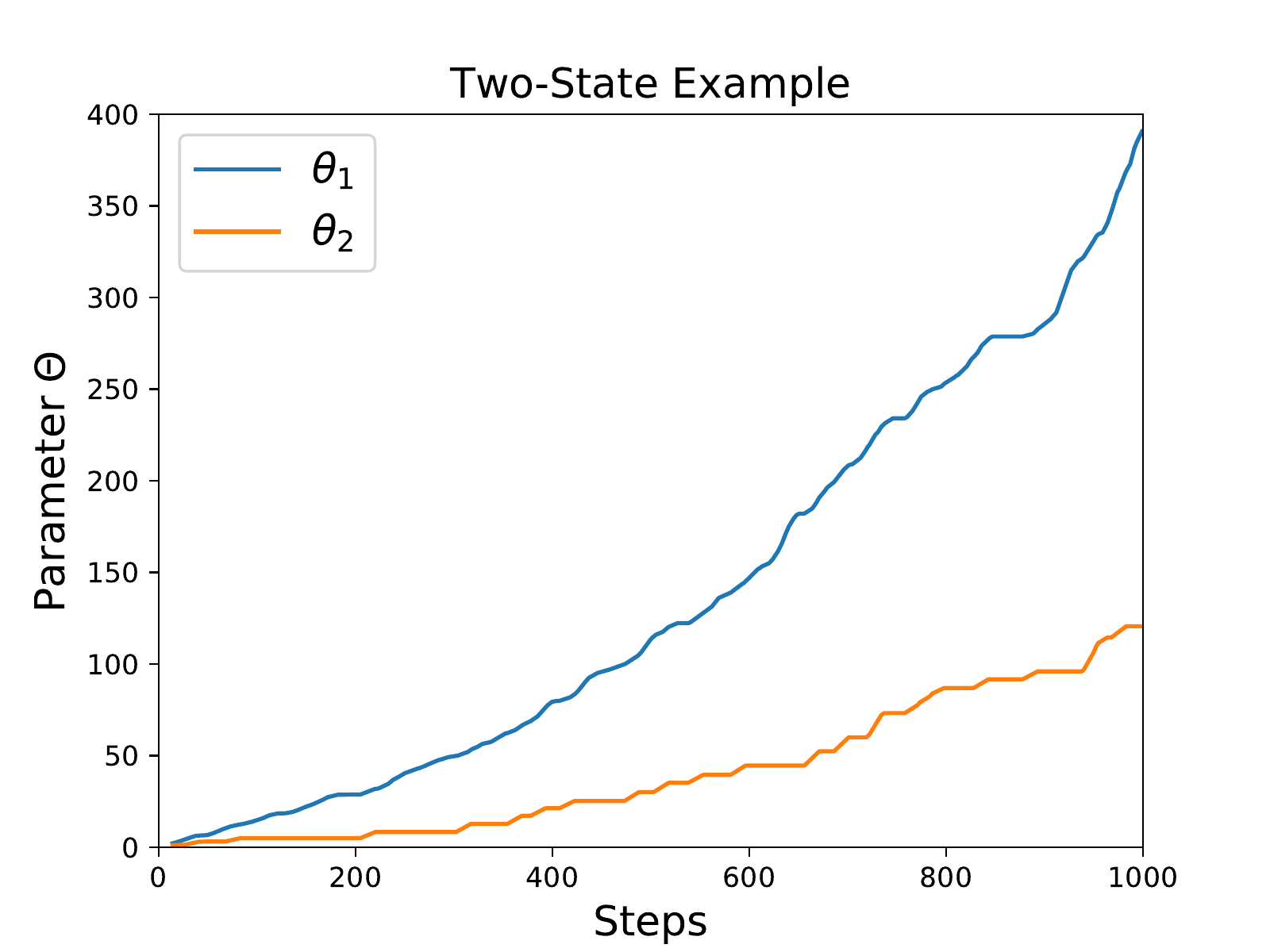}
    \caption{Demonstration of instability on the counter example. The components of parameter $\theta$ are shown in the figure. 
        The initial weights $\theta=(2,0)^{\top}$, $\gamma=0.99$, $\lambda=0.99$. We run $\sigma$ from 0 to 1 with step-size 0.01, and all the solutions are similar. We only show one result here.}
\end{figure}

The Figure 2 shows the numerical solution of the parameter $\theta$ learned by (\ref{Eq:semi-gradient}) for the counterexample (Figure1).
This simple example is very striking because the full-sampling and pure-expectation methods are arguably the simplest
and best-understood methods and the linear, semi-gradient method is arguably the simplest and best-understood kind of function approximation.
This result
shows that even the simplest combination of full-sampling and pure-expectation with function approximation can be unstable if the updates are not done according to the on-policy distribution.

\section{Gradient Q$(\sigma, \lambda)$}

We have discussed the divergence of Q$(\sigma,\lambda)$ with semi-gradient method. In this section, 
we propose a convergent and stable TD algorithm: gradient Q$(\sigma,\lambda)$.

\subsection{Objective Function}
We derive the algorithm via MSPBE \cite{sutton2009fast_a}:
\begin{flalign}
\nonumber
\text{MSPBE}(\theta,\lambda)&=\frac{1}{2}\|\Phi\theta-\Pi\mathcal{B}^{\pi,\mu}_{\sigma,\lambda}(\Phi\theta)\|^{2}_{\Xi},\\
\label{Eq:solve_mspbe}    
\theta^{*}&=\arg\min_{\theta}\text{MSPBE}(\theta,\lambda),
\end{flalign}
where $\Pi = \Phi(\Phi^{\top}\Xi\Phi)^{-1}\Phi^{\top}\Xi$ is an $|\mathcal{S}| \times |\mathcal{S}|$ \emph{projection matrix} which projects any value function into the space generated by $\Phi$.
After some simple algebra, we can further rewrite MSPBE$(\theta,\lambda)$ as a standard weight least-squares equation:
\begin{flalign}
\label{Eq:mspbe}
\text{MSPBE}(\theta,\lambda)=\frac{1}{2}\|A_{\sigma}\theta+b_{\sigma}\|^{2}_{M^{-1}},
\end{flalign}
where $M=\mathbb{E}[\phi_{k}\phi_{k}^{\top}]=\Phi^{\top}\Xi\Phi$.

Now, We define the update rule as follows, for a given trajectory $\{(S_{k},A_{k},R_{k},S_{k+1})\}_{k\ge0}$, $\forall\lambda,\sigma\in[0,1]$:
\begin{flalign}
\nonumber
&e_{k}=\phi_{k}+\gamma\lambda e_{k-1},\\
\nonumber
&\delta_{k}=R_{k}+\gamma\theta_{k}^{\top}(\sigma\phi_{k+1}+(1-\sigma)\mathbb{E}_{\pi}\phi(S_{k+1},\cdot))-\theta_{k}^{\top}\phi_{k},\\
\label{theata_update}
&\theta_{k+1}=\theta_{k}-\alpha_{k}\frac{1}{2}\nabla_{\theta} \text{MSPBE}(\theta,\lambda)|_{\theta=\theta_{k}},
\end{flalign}
where $\alpha_{k}>0$ is step-size, $\phi_{k}=\phi(S_{k},A_{k})$, $e_{k}$ is trace vector and $e_{-1}=0$.
By Eq.(\ref{theata_update}),
it is worth notice that the challenges of solving (\ref{Eq:solve_mspbe}) are two-fold:
\begin{itemize} 
\item
The computational complexity of the invertible matrix $M^{-1}$ is at least $\mathcal{O}(p^3)$~\cite{golub2012matrix}, where $p$ is the dimension of feature space. Thus, it is too expensive to use gradient to solve the problem (\ref{Eq:solve_mspbe}) directly.
\item
Besides, as pointed out by Szepesv{\'a}~\shortcite{szepesvari2010algorithms} and Liu et al.~\shortcite{liu2015finite}, we cannot get an unbiased estimate of $\nabla_{\theta}\text{MSPBE}(\theta,\lambda)=A^{\top}_{\sigma}M^{-1}(A_{\sigma}\theta+b_{\sigma})$. 
In fact, since the update law of gradient involves the product of expectations, 
the unbiased estimate cannot be obtained via a single sample. It needs to sample twice, which is a double sampling problem. Secondly, $M=\mathbb{E}[\phi_{t} \phi_{t}^\top]^{-1}$ cannot also be estimated via a single sample, which is the second bottleneck of applying stochastic gradient to solve problem (\ref{Eq:solve_mspbe}).
\end{itemize}
We provide 
a practical way to solve the above problem in the next subsection.

\subsection{Algorithm Derivation}
The gradient $\nabla_{\theta} \text{MSPBE}(\theta,\lambda)$ in Eq.(\ref{theata_update}) can be replaced by the following equation:
\begin{flalign}
\nonumber
&\frac{1}{2}\nabla_{\theta} \text{MSPBE}(\theta,\lambda)\\
\label{gradient_equal}
=&\nabla_{\theta}\mathbb{E}[\delta_{k}e_{k}]^{\top}\underbrace{\mathbb{E}[\phi_{k}\phi^{\top}_{k}]^{-1}\mathbb{E}[\delta_{k}e_{k}]}_{\overset{\text{def}}=\omega(\theta_{k})}.
\end{flalign}
The proof of Eq.(\ref{gradient_equal}) is similar to the derivation in Chapter 7 of \cite{maei2011gradient}, thus we omit its proof.
Furthermore, the following Proposition \ref{prop1} provides a new way to estimate $\nabla_{\theta} \text{MSPBE}(\theta,\lambda)$.
\begin{proposition}
    \label{prop1}
    Let $e_{t}$ be the eligibility traces vector that is generated as $e_{k}=\lambda\gamma e_{k-1}+\phi_{k}$, let
    \begin{flalign}
\label{Delta}
\Delta_{k,\sigma}&=\gamma\{\sigma\phi_{k+1}+(1-\sigma)\mathbb{E}_{\pi}\phi(S_{k+1},\cdot)\}-\phi_{k},\\
\nonumber
v_{\sigma}(\theta_{k})&=(1-\sigma)\{\mathbb{E}_{\pi}\phi(S_{k+1},\cdot)-\lambda\phi_{k+1}\}e^{\top}_{k}\\
\label{vk}
&\hspace{1cm}+\sigma(1-\lambda)\phi_{k+1}e^{\top}_{k},
\end{flalign}
then the following holds,
\begin{flalign}
\nonumber
\theta_{k+1}&=\theta_{k}-\alpha_{k}\frac{1}{2}\nabla_{\theta} \emph{MSPBE}(\theta,\lambda)|_{\theta=\theta_{k}}\\
\nonumber
&=\theta_{k}-\alpha_{k}\mathbb{E}[\Delta_{k,\sigma}e_{k}^{\top}]\omega(\theta_{k})\\
\label{theta_uptate_2}
&=\theta_{k}+\alpha_{k}\{\mathbb{E}[\delta_{k}e_{k}]-\gamma\mathbb{E}[v_{\sigma}(\theta_{k})]\omega(\theta_{k})\}.
\end{flalign}
\end{proposition}
\begin{proof}
See Appendix B.
\end{proof}

It is too expensive to calculate inverse matrix $\mathbb{E}[\phi_{k}\phi^{\top}_{k}]^{-1}$ in Eq.(\ref{gradient_equal}).
In order to develop an efficient $\mathcal{O}(p)$ algorithm, Sutton et al.\shortcite{sutton2009fast_a} use a weight-duplication trick. They propose the way to estimate $\omega(\theta_{k})$ on a fast timescale:
\begin{flalign}
\label{gq_update2}
\omega_{k+1}=\omega_{k}+\beta_{k}(\delta_{k}e_{k}-\phi_k\omega^{\top}_{k}\phi_k).
\end{flalign}
Now, sampling from Eq.(\ref{theta_uptate_2}) directly , we define the update rule of $\theta$ as follows,
\begin{flalign}
\label{gq_update1}
\theta_{k+1}=\theta_{k}+\alpha_{k}(\delta_{k}e_{k}-\gamma v_{\sigma}(\theta_{k})\omega_{k})
\end{flalign}
where $\delta_{k},e_{k}$ is defined in Eq.(\ref{theata_update}), $v_{\sigma}(\theta_{k})$ is defined in Eq.(\ref{vk}) and $\alpha_{k},\beta_{k}$ are step-size.
More details of gradient Q$(\sigma,\lambda)$ are summary in Algorithm \ref{alg:algorithm1}.
\begin{algorithm}[tb]
    \caption{Gradient Q$(\sigma,\lambda)$}
    \label{alg:algorithm1}
    \begin{algorithmic}
        \STATE { \textbf{Require}:Initialize parameter $\omega_{0},v_{0}=0$, ${\theta}_{0}$ arbitrarily, $\alpha_{k}>0,\beta_{k}>0$}. \\
        \STATE { \textbf{Given}: target policy $\pi$, behavior policy $\mu$.}\\
        \FOR{$i=0$ {\bfseries to} $n$}
        \STATE ${e}_{-1}={0}$.
        \FOR{$k=0$ {\bfseries to} $T_{i}$}
        \STATE Observe $\{S_{k},A_{k},R_{k+1},S_{k+1}\}$ by $\mu$.
        \STATE {\color{blue}{\# Update traces }}
        \STATE ${e}_{k}=\lambda\gamma {e}_{k-1}+{\phi}_{k}$.
        \STATE $\delta_{k}=R_{k}+\gamma\{\sigma{\theta}^{\top}_{k}{\phi}(S_{k+1},A_{k+1})$
        \STATE$\hspace{0.7cm}+(1-\sigma){\theta}^{\top}_{k}\mathbb{E}_{\pi}{\phi}(S_{k+1},\cdot)\}-{\theta}^{\top}_{k}\phi(S_{k},A_{k})$.
        \STATE {\color{blue}{\# Update parameter}}
        \STATE $v_{k+1}=\sigma(1-\lambda)\phi(S_{k+1},A_{k+1})e^{\top}_{k}$
        \STATE \hspace{0.7cm}$+(1-\sigma)\{\mathbb{E}_{\pi}\phi(S_{k+1},\cdot)-\lambda\phi(S_{k+1},A_{k+1})\}e^{\top}_{k}.$
        \STATE${\theta}_{k+1}=\theta_{k}+\alpha_{k}(\delta_{k}e_{k}-\gamma v_{k}\omega_{k})$.
        \STATE$\omega_{k+1}=\omega_{k}+\beta_{k}\{\delta_{k}e_{k}-\phi(S_{k},A_{k})\omega^{\top}_{k}\phi(S_{k},A_{k})\}.$
        \ENDFOR
        \ENDFOR
        \STATE { \textbf{Output}:${\theta}$}
    \end{algorithmic}
\end{algorithm}

\subsection{Convergence Analysis}

We need some additional assumptions to present the convergent of Algorithm \ref{alg:algorithm1}.
\begin{assumption}
    \label{ass:positive_lr}
    The positive sequence $\{\alpha_{k}\}_{k\ge0}$, $\{\beta_{k}\}_{k\ge0}$ satisfy $\sum_{k=0}^{\infty}\alpha_{k}=\sum_{k=0}^{\infty}\beta_{k}=\infty,\sum_{k=0}^{\infty}\alpha^{2}_{k}<\infty,\sum_{k=0}^{\infty}\beta^{2}_{k}<\infty$
    with probability one.
\end{assumption}

\begin{assumption}[Boundedness of Feature, Reward and Parameters\cite{liu2015finite}]
    \label{ass:boundedness}
    (1)The features $\{\phi_{t}, \phi_{t+1}\}_{t\ge0}$ have uniformly bounded second moments, where $\phi_{t}=\phi(S_{t}),\phi_{t+1}=\phi(S_{t+1})$.
    (2)The reward function has uniformly bounded second moments.
    (3)There exists a bounded region $D_{\theta}\times D_{\omega}$, such that $\forall (\theta,\omega)\in D_{\theta}\times D_{\omega}$.
\end{assumption}
Assumption \ref{ass:boundedness} guarantees that the matrices $A_{\sigma}$ and $M$, and vector $b_{\sigma}$ are uniformly bounded.
After some simple algebra, we have $A_{\sigma}=\mathbb{E}[e_{k}\Delta_{k,\sigma}]$, see Appendix C.
The following Assumption \ref{invA} implies $A_{\sigma}^{-1}$ is well-defined.
\begin{assumption}
    \label{invA}
    $A_{\sigma}=\mathbb{E}[e_{k}\Delta_{k,\sigma}]$ is non-singular, where $\Delta_{k,\sigma}$ is defined in Eq.(\ref{Delta}).
\end{assumption}

\begin{theorem}[Convergence of Algorithm \ref{alg:algorithm1}]
    \label{theorem2}
    Consider the iteration $(\theta_{k},\omega_k)$ generated by (\ref{gq_update1}) and (\ref{gq_update2}), if $\beta_{k}=\eta_{k}\alpha_{k}$, $\eta_{k}\rightarrow0$, as $k\rightarrow\infty$ , and $\alpha_{k},\beta_k$ satisfies Assumption \ref{ass:positive_lr}. 
    The sequence $\{(\phi_{k},R_{k},\phi_{k+1})\}_{k\ge0}$ satisfies Assumption \ref{ass:boundedness}.
    Furthermore, $A_{\sigma}=\mathbb{E}[e_{k}\Delta_{k,\sigma}]$ satisfies Assumption \ref{invA}.
        Let
        \begin{flalign}
        \label{def:G}
        G(\omega,\theta)&=A^{\top}_{\sigma}M^{-1}(A_{\sigma}\theta+b_{\sigma}),\\
         \label{def:H}
    H(\omega,\theta)&=A_{\sigma}\theta+b_{\sigma}-M \omega.
    \end{flalign}
    Then $(\theta_{k},\omega_k)$ converges to $(\theta^{*},\omega^{*})$ with probability one,
    where $(\theta^{*},\omega^{*})$ is the unique global asymptotically stable equilibrium w.r.t ordinary differential equation (ODE)
    $\dot\theta(t)=G(\Omega(\theta),\theta), \dot\omega(t)=H(\omega,\theta)$ correspondingly, and $\Omega(\theta)$: $ \theta\mapsto M^{-1}(A_{\sigma}\theta+b_\sigma)$.
\end{theorem}

\begin{proof}
The ODE method (see Lemma 1; Appendix D) is our main tool to prove Theorem \ref{theorem2}.
    Let 
    \begin{flalign}
    \label{M_k,N_k}
    M_{k}&=\delta_{k}e_{k}-\gamma v_{\sigma}(\theta_{k})\omega_{k}-M^{-1}(A_{\sigma}\theta_{k}+b_{\sigma}), \\
    N_{k}&=(\hat{A}_{k}-A_{\sigma})\theta_k+(\hat{b}_{k}-b_{\sigma})-(\hat{M}_{k}-M)\omega_{k}.
    \end{flalign}
Then, we rewrite the iteration (\ref{gq_update1}) and (\ref{gq_update2}) as follows,
    \begin{flalign}
    \label{gq_update1_1}
    \theta_{k+1}&=\theta_{k}+\alpha_{k}(M^{-1}(A_{\sigma}\theta_k+b_{\sigma})+M_k),\\
    \label{gq_update2_1}
    \omega_{k+1}&=\omega_{k}+\beta_{k}(H(\theta_{k},\omega_{k})+N_{k}).
    \end{flalign} 
    The Lemma 1 requires us to verify the following 4 steps.

    \underline{\textbf{Step 1: (Verifying the condition A2)}} \emph{Both of the functions $G$ and $H$ are Lipschitz functions.}

    By Assumption \ref{ass:boundedness}-\ref{invA}, $A_{\sigma}$, $b$ and $M$ are uniformly bounded, thus it is easy to check $G$ and $H$ are Lipschitz functions.

    \underline{\textbf{Step 2: (Verifying the condition A3)} } 
    \emph{Let the $\sigma$-field $\mathcal{F}_{k}=\sigma\{\theta_{t},\omega_{t};t\leq k\}$, then $\mathbb{E}[M_{k}|\mathcal{F}_k]=\mathbb{E}[N_{k}|\mathcal{F}_k]=0.$ 
    Furthermore, there exists non-negative a constant $K > 0$, s.t. $\{M_{k}\}_{k\in\mathbb{ N}}$ and$ \{N_{k}\}_{k\in\mathbb{ N}}$ are square-integrable with
    \begin{flalign}
    \label{MN}
    \mathbb{E}[\|M_{k}\|^{2}|\mathcal{F}_{k}],\mathbb{E}[\|N_{k}\|^{2}|\mathcal{F}_{k}]\leq K(1+\|\theta_{k}\|^{2}+\|\omega_{k}\|^{2}).
    \end{flalign}
   }

    \begin{figure*}[t]
    \centering
    \subfigure
    {\includegraphics[width=5.8cm,height=4.2cm]{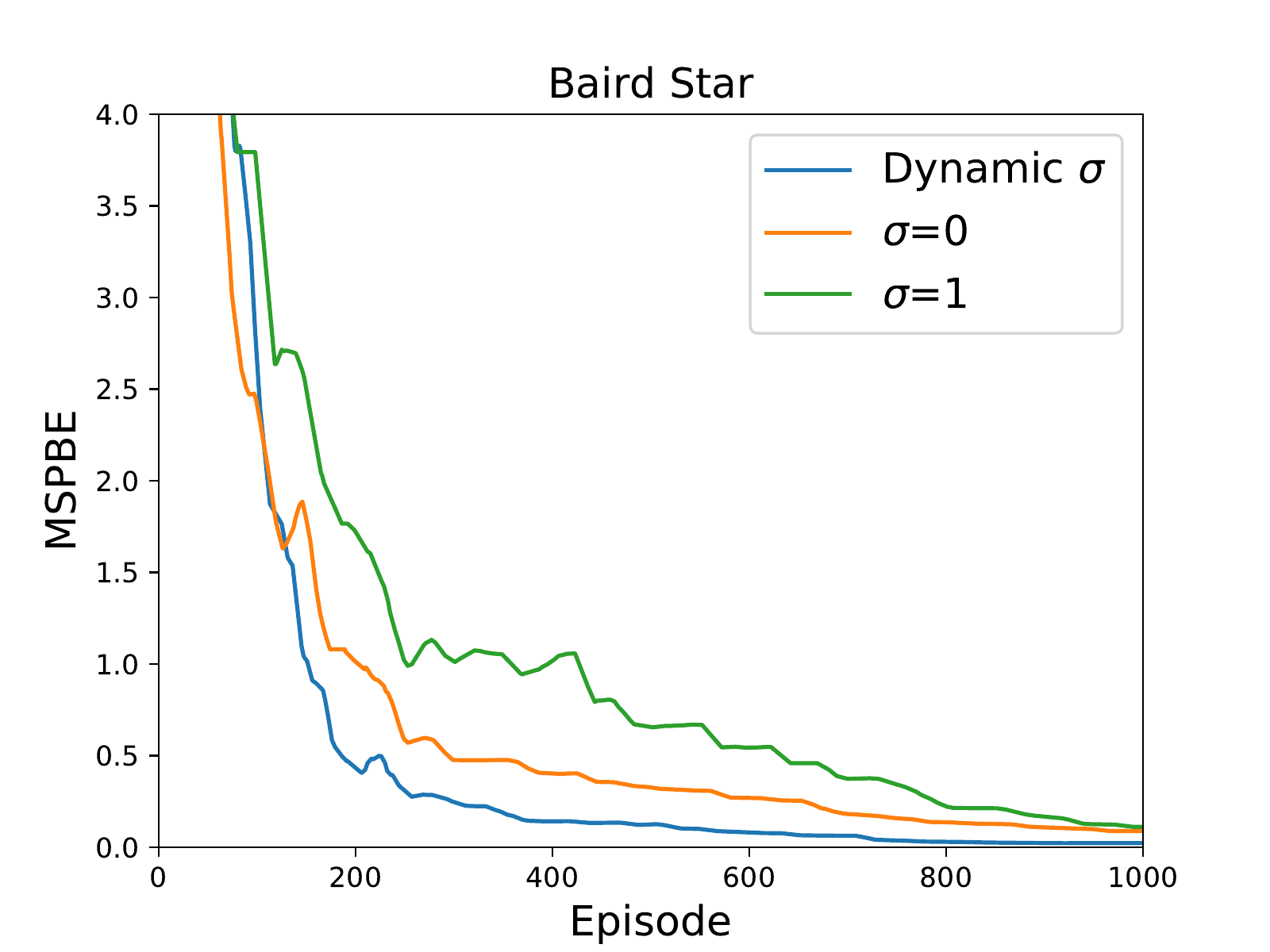}}
    \subfigure
    {\includegraphics[width=5.8cm,height=4.2cm]{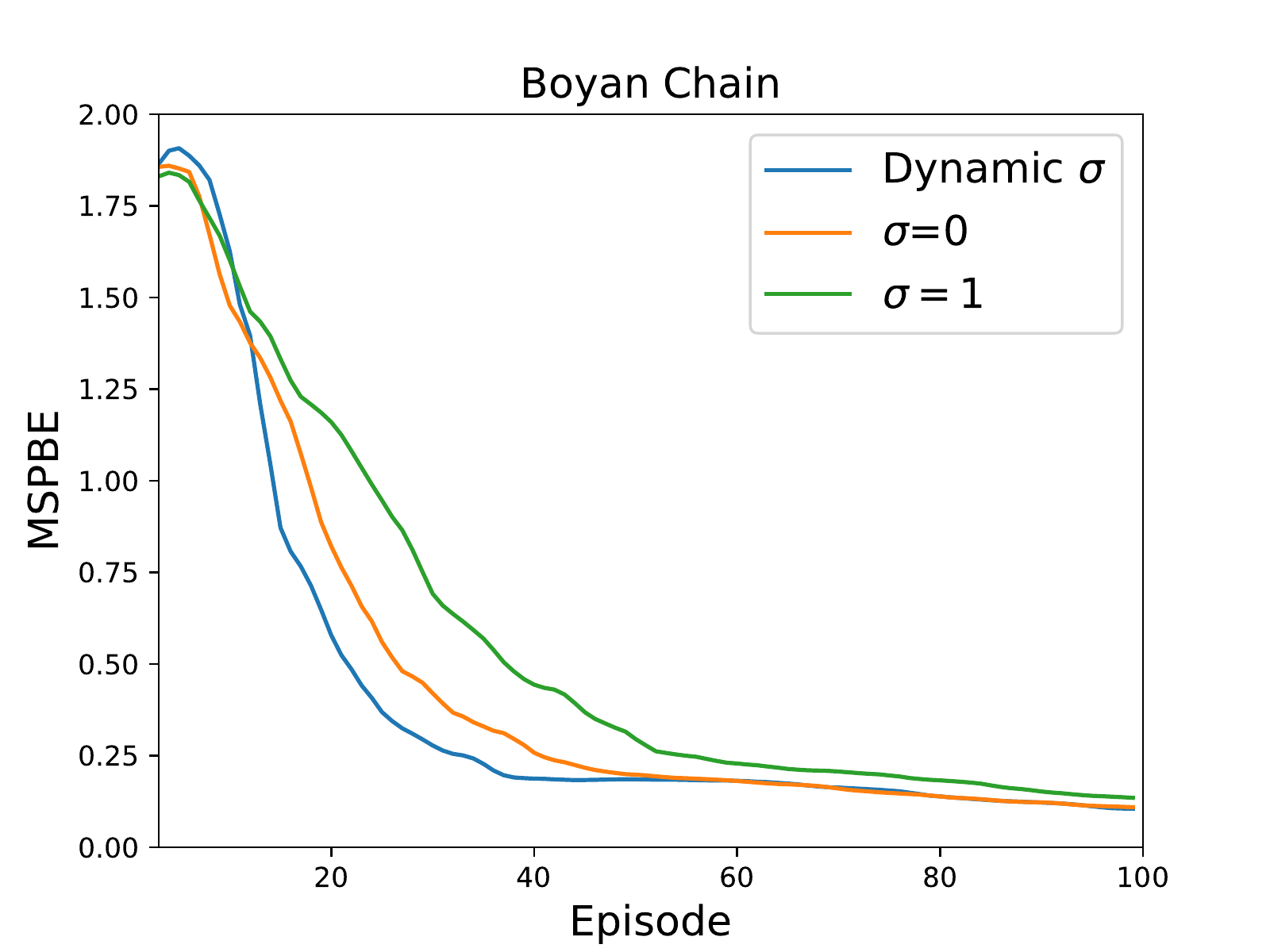}}
    \subfigure
    {\includegraphics[width=5.8cm,height=4.2cm]{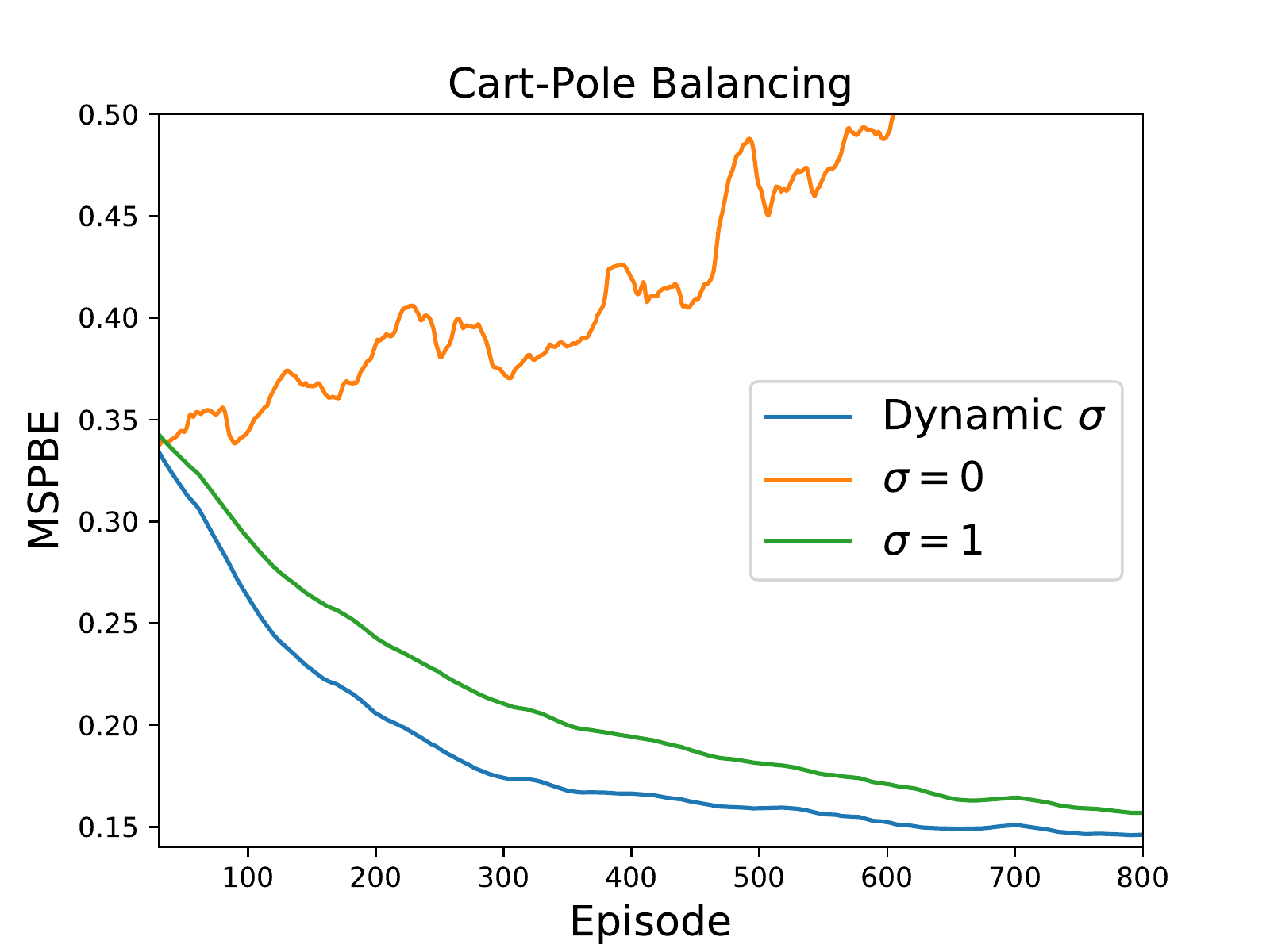}}
    \caption
    {
        MSPBE comparison with different $\sigma$: $\sigma= 0$ (pure-expectation) , $\sigma= 1$ (full-sampling). Dynamic $\sigma\in (0,1)$.
    }
\end{figure*}
    
    By Eq.(\ref{theta_uptate_2}), $\mathbb{E}[M_{k}|\mathcal{F}_k]=0$. 
    With $\mathbb{E}[\hat{A}_{k}]=A_{\sigma}, \mathbb{E}[\hat{b}_{k}]=b_{\sigma}, \mathbb{E}[\hat{M}_{k}]=M,$ we have
    $\mathbb{E}[N_{k}|\mathcal{F}_k]=0$.
    By Assumption \ref{ass:boundedness}, Eq.(\ref{A_k}) and Eq.(\ref{b_k}), there exists non-negative 
    constants $K_1,K_2,K_3$ such that $\{\|\hat{A}_{k}\|^{2},\|A_{\sigma}\|^{2}\}\leq K_1$,
    $\{\|\hat{b}_{k}\|^{2},\|b_{\sigma}\|^{2}\}\leq K_2$, $\{\|\hat{M}_{k}\|^{2},\|M\|^{2}\}\leq K_3$, which implies all above terms are bounded.
    Thus, there exists a non-negative 
    constant $\widetilde{K}_1$ s.t the following Eq.(\ref{b}) holds,
    \begin{flalign} 
    \nonumber
    &\mathbb{E}[\|N_{k}\|^{2}|\mathcal{F}_{k}]\\
     \nonumber
    \leq&\mathbb{E}\Big[(\|(\hat{A}_{k}-A_{\sigma})\theta_{k}\|+\|\hat{b}_{k}-b_{\sigma}  \|  +  \| (\hat{M}_{k}-M)\omega_{k}\|)^{2}\Big|\mathcal{F}_{k}\Big]\\
    \label{b}
\leq& {\widetilde{K}_1}^{2}(1+\|\theta_{k}\|^{2}+\|\omega_{k}\|^{2}).
    \end{flalign}
Similarly, by Assumption \ref{ass:boundedness}, $R_{k},\phi_k$ have uniformly bounded second,
then $\mathbb{E}[\|M_{k}\|^{2}|\mathcal{F}_{k}]\leq{\widetilde{K}_2}^{2}(1+\|\theta_{k}\|^{2}+\|\omega_{k}\|^{2})$ holds for a constant $\widetilde{K}_2>0$.
Thus, Eq.(\ref{MN}) holds.

\underline{\textbf{Step 3: (Verifying the condition A4)} } 
\emph{For each $\theta\in\mathbb{R}^{p}$, the ODE
    $
    \dot \omega(t) = H(\omega(t), \theta)
    $
    has a unique global asymptotically stable equilibrium $\Omega(\theta)$ such that:
    $\Omega(\theta):\mathbb{R}^{p}\rightarrow\mathbb{R}^{k}$ is Lipschitz.}

For a fixed $\theta$, let
$H_{\infty}(\omega,\theta)=\lim_{r\rightarrow\infty}\dfrac{H(r\omega(t),\theta)}{r}=M\omega({t}).$
We consider the ODE
\begin{flalign}
\label{ode:h-infty}
\dot{\omega}(t)=H_{\infty}(\omega,\theta)=M\omega(t).
\end{flalign}
Assumption \ref{invA} implies that $M$ is a positive definite matrix, thus,
for ODE (\ref{ode:h-infty}), origin is a globally asymptotically stable equilibrium.
Thus, for a fixed $\theta$, by Assumption \ref{invA}, 
\begin{flalign}
\label{def:oemga-star}
\omega^{*}=M^{-1}(A_{\sigma}\theta+b_{\sigma})
\end{flalign}
is the unique globally asymptotically stable equilibrium of ODE
$\dot \omega(t) = A_{\sigma}\theta+b_{\sigma}-M\omega(t)\overset{(\ref{def:H})}=H(\omega(t), \theta).$
Let $\Omega(\theta): \theta\mapsto M^{-1}(A_{\sigma}\theta+b_{\sigma}),$
it is obvious $\Omega$ is Lipschitz.


\underline{\textbf{Step 4: (Verifying the condition A5)} } \emph{The ODE $\dot \theta(t) =G\big(\Omega(\theta(t)),\theta(t)\big)$ has a unique global asymptotically stable equilibrium $\theta^{*}$.}

Let
$
    G_{\infty}(\theta)=\lim_{r\rightarrow\infty}\frac{G(r\theta,\omega)}{r}
=A^{\top}_{\sigma} M^{-1} A_{\sigma}\theta$.
We consider the following ODE
\begin{flalign}
    \label{ode-G}
    \dot{\theta}(t)=G_{\infty}(\theta(t)).
\end{flalign}
By Assumption \ref{ass:boundedness}-\ref{invA}, $A_\sigma$ is invertible and $M^{-1}$ is positive definition, thus $A^{\top}_{\sigma} M^{-1} A_{\sigma}$ is a positive defined matrix.
Thus the ODE (\ref{ode-G}) has unique global asymptotically stable equilibrium: origin point.
Now, let's consider the iteration (\ref{gq_update1})/(\ref{gq_update1_1})
associated with the ODE
$
\dot{\theta}(t)=(\gamma\mathbb{E}[\sigma\phi_{k+1}+(1-\sigma)\mathbb{E}_{\pi}[\phi(S_{k+1},\cdot)]]e_{k}^{\top}M^{-1}-I)\mathbb{E}[\delta_{k}e_{k}|\theta(t)]
,$ which can be rewritten as follows,
\begin{flalign}
\label{ode-theta-1}
    \dot{\theta}(t)&=\mathbb{E}[\Delta_{k,\sigma}e_{k}^{\top}]M^{-1}(A_{\sigma}\theta(t)+b_\sigma)
    \\
    \label{ode-theta-2}
    &=A^{\top}_{\sigma}M^{-1}(A_{\sigma}\theta(t)+b_\sigma)
    \overset{( \ref{def:G})}=G(\theta(t)),
\end{flalign}
Eq.(\ref{ode-theta-2}) holds due to $\mathbb{E}[\delta_{k}e_{k}|\theta(t)]=A_{\sigma}\theta(t)+b_\sigma$.
By Assumption \ref{invA}, $A_\sigma$ is invertible, then 
\begin{flalign}
\label{def:theta-star}
\theta^{*}=-A_{\sigma}^{-1}b_{\sigma}
\end{flalign} is the unique global asymptotically stable equilibrium of ODE (\ref{ode-theta-2}). 

Since then, we have verified all the conditions of Lemma 1, thus the following almost surely 
\[
(\theta_{k},\omega_k)~\rightarrow~~(\theta^{*},\omega^{*}), ~~\text{as}~~
k\rightarrow\infty,
\]
where $\omega^{*}$ is defined in (\ref{def:oemga-star}), $\theta^{*}$ is defined in (\ref{def:theta-star}).
\end{proof}

\section{Experiment}

In this section, we test both policy evaluation and control capability of the proposed GQ$( \sigma,\lambda)$ algorithm and validate the trade-off between full-sampling and pure-expectation on some standard domains.
 In this section, for all experiments, we set the hyper parameter $\sigma$ as follows, $\sigma\sim\mathcal{N}(\mu, 0.01^2)$, where $\mu$ ranges dynamically from 0.02 to 0.98 with step of 0.02, and $\mathcal{N}(\cdot,0.01^2)$ is Gaussian distribution with standard deviation $0.01$. In the following paragraph, 
we use the term \emph{dynamic} $\sigma$ to represent the above way to set $\sigma$.

\subsection{Policy Evaluation Task}

We employ three typical domains in RL for policy evaluation:
Baird Star \cite{baird1995residual}, Boyan Chain \cite{boyan2002technical} and linearized Cart-Pole balancing.

\textbf{Domains}~ Baird Star is a well known example for divergence in off-policy TD learning, which considers 
$7$-state $\mathcal{S}=\{\mathtt{s}_1,\cdots,\mathtt{s}_7\}$ and
$\mathcal{A}$= \{$\mathtt{dashed}$, $\mathtt{solid}$\}.
The behavior policy $\mu$ selects the $\mathtt{dashed}$ and $\mathtt{solid}$ actions
with $\mu(\mathtt{dashed}|\cdot)=\frac{6}{7}$ and $\mu(\mathtt{solid}|\cdot)=\frac{1}{7}$.
The target policy always takes the $\mathtt{solid}$: $\pi(\mathtt{solid}|\cdot)=1$.

The second benchmark MDP is the classic chain example from \cite{boyan2002technical},
which considers a chain of $14$ states $\mathcal{S} = \{\mathtt{s}_1, \cdots, \mathtt{s}_{14}\}$. 
Each transition from state $\mathtt{s}_i$ results in state $\mathtt{s}_{i+1}$ or $\mathtt{s}_{i+2}$ with equal probability and a reward of $-3$.
The behavior policy we chose is random.

For the limitation of space, we provide more details of the dynamics of Boyan Chain and Baird Star, chosen policy and features in Appendix E.

Cart-Pole balancing is widely used for many RL tasks.
According to ~\cite{dann2014policy}, the target policy we use in this section is the optimal policy
$\pi^{*}(a|s)= \mathcal{N}(a|{\beta_{1}^{*}}^{\top}s,(\sigma_{1}^{*})^{2})$,
where the hyper parameters $\beta_{1}^{*}$ and $\sigma_{1}^{*}$ are computed using dynamic programming. 
The feature chosen according to ~\cite{dann2014policy} is a imperfect feature set : $\phi(s)=(1,s_{1}^{2},s_{2}^{2},s_{3}^{2},s_{4}^{2})^{\top}, ~\text{where}~s=(s_1,s_2,s_3,s_4)^{\top}.$

\textbf{Performance Measurement}~In this section,
we use 
the empirical $\text{MSPBE}=\frac{1}{2}\|\hat{b}_{\sigma}-\hat{A}_{\sigma}\theta\|^{2}_{\hat{M}^{-1}}$
to evaluate the performance of all the algorithms, where we evaluate $\hat{A}_\sigma$, $\hat{b}_\sigma$, and $\hat{M}$ according to their unbiased estimates (\ref{A_k}), $(\ref{b_k})$ and $\phi_{k}\phi^{\top}_{k}$, the features are presented in Appendix E.

\textbf{Results Report}
Figure 3 shows the results with dynamic $\sigma$ achieves the best performance on all the three domains. 
Our results show that an intermediate value of $\sigma$, which results in a mixture of the full-sampling and pure-expectation algorithms,  performs better than either extreme ($\sigma=0$ or $1$). 
This validates the trade-off between full-sampling and pure-expectation for policy evaluation in standard domains, which also give us some insights that 
unifying some disparate existing method can create a better performing algorithm.

\subsection{Control Task}
In this section, we test the control capability of GQ$( \sigma,\lambda)$ algorithm on mountain car domain, where the agent considers the task of driving an underpowered car up a steep mountain road. 
The agent receives a reward of $-1$ at every step until it reaches the goal region at the top of the hill.
Since the state space of this domain is continuous, we use the open tile coding software\footnote{\url{http://incompleteideas.net/rlai.cs.ualberta.ca/RLAI/RLtoolkit/tilecoding.html}} to extract feature of states. 

\textbf{Empirical Performance Comparison}
The performance shown in Figure 4 is an average of 100 runs, and each run contains 400 episodes. 
We set $\lambda=0.99$, $\gamma=0.99$, and step-size $\alpha_{k}=\{10^{-2},2\times 10^{-2},10^{-3},2\times 10^{-3}\}$, $
\eta_{k}=\{2^{0},2^{-1},\cdots,2^{-10}\}$.

\begin{figure}[t]
    \centering
    \subfigure[]
    {\includegraphics[width=4.1cm,height=3.2cm]{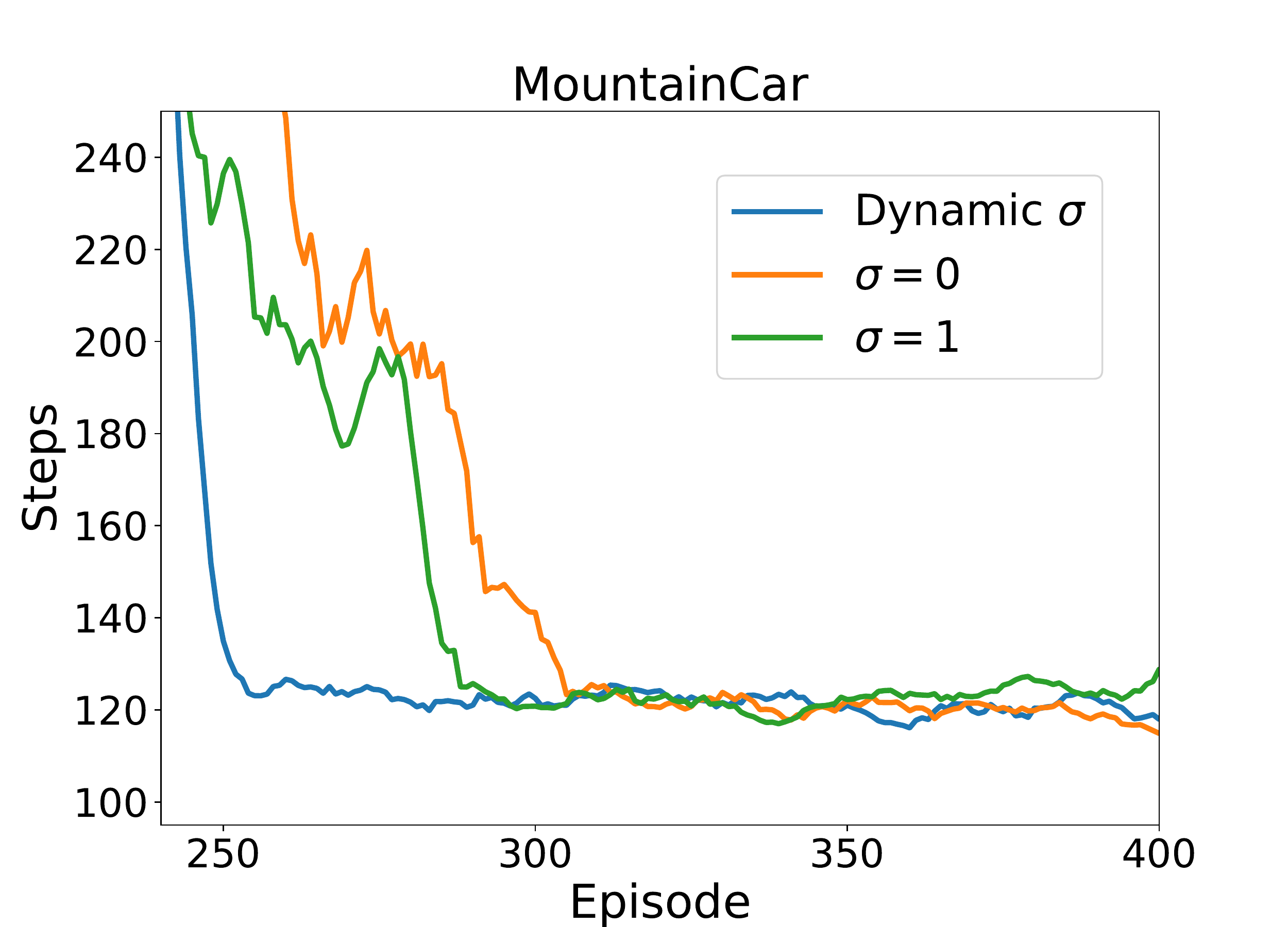}}
    \subfigure[]
    {\includegraphics[width=4.1cm,height=3.2cm]{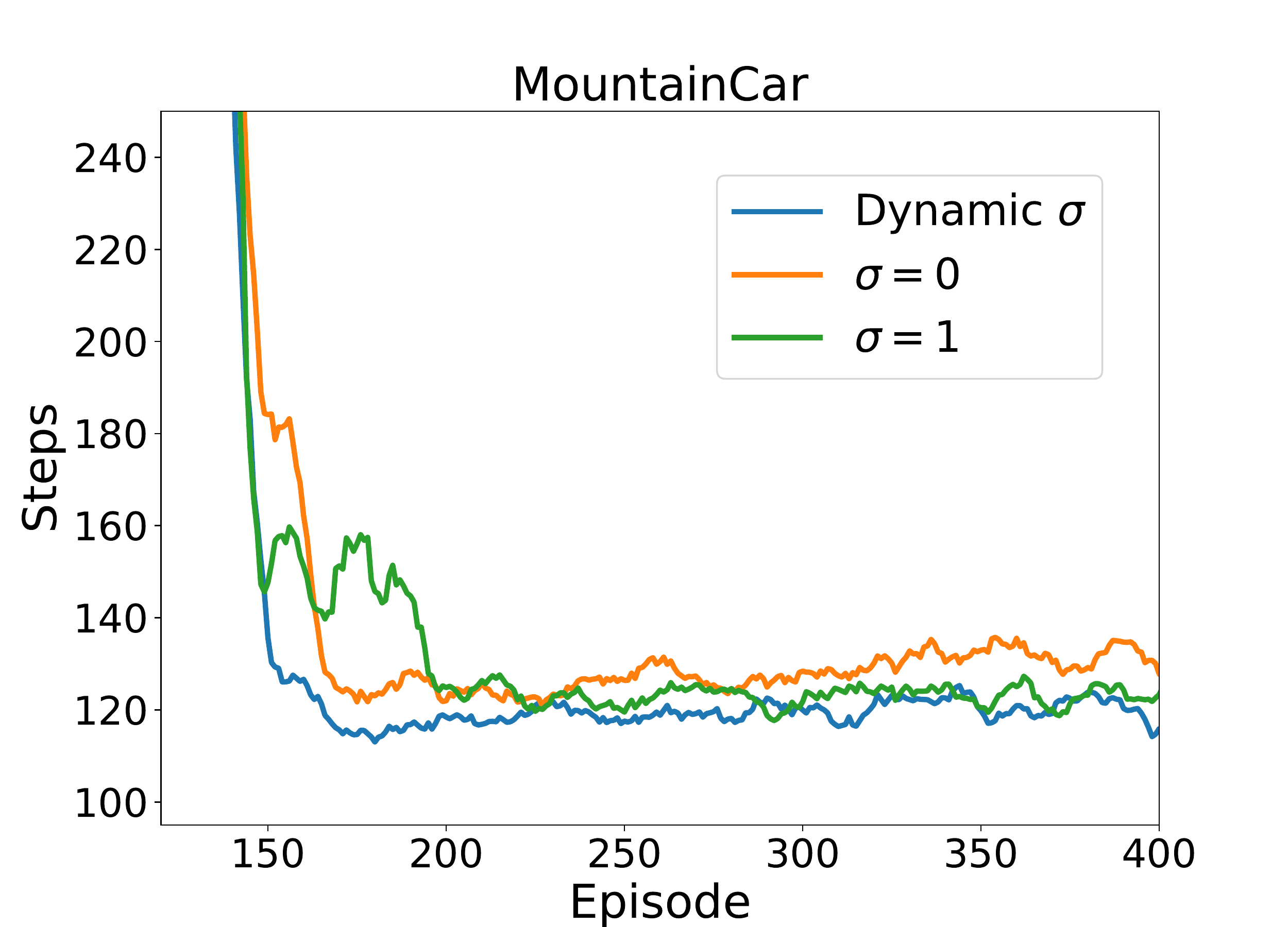}}
    \caption
    {
        Comparison of empirical performance with different step-size: (a) $\alpha_{k}=0.001$, (b) $\alpha_{k}=0.002$.
    }
\end{figure}

 The result in Figure 4 shows that GQ($\sigma,\lambda$) with an intermediate $\sigma$ (between $0$ and $1$) has a better performance than the extreme case ($\sigma=0~\text{and}~1$).
 This experiment further validates that unifying some existing algorithms can create a better algorithm for RL.

\textbf{Variance Comparison}
Now, we investigate and survey the variance of the performance of GQ($\sigma,\lambda$) during training.
All parameters are set as before.
We run the size of feature maps $p$ from $2^{7}$ to $2^{11}$. The outcomes related to different $p$ are very similar, so we only show the results of $p=1024$. 

Result in Figure 5 shows that the performance with the least variance is neither $\sigma=0$ nor $\sigma=1$, but the dynamic $\sigma$ reaches the least variance.

\textbf{Overall Presentation} Now, we give more comprehensive results of the trade-off between full-sampling and pure-expectation.

\begin{table}[H]
\centering
\begin{tabular}{|c|c|c|c|}
\hline
Case&I     & II        & III  \\ \hline
Percentage&\textbf{42.2\%} & \textbf{24.1\%} & 33.7\% \\ \hline
\end{tabular}
\label{tab:mc-all}
\caption{Overall data statistics.}
\end{table}

 We statistic of the number of $\sigma$ happens for the following three case, 
 (I) GQ($\sigma,\lambda$) performs better than both $\sigma=0$ and $\sigma=1$.
(II) GQ($\sigma,\lambda$) performs better than one of  $\sigma=0$ and $\sigma=1$.
(III) GQ($\sigma,\lambda$) performs worse than both $\sigma=0$ and $\sigma=1$.
The setting of $\sigma$ is the same as the previous section, and the total number of $\sigma$ reaches $51$.

\begin{figure}[t]
    \centering
    {\includegraphics[width=6.5cm,height=5.0cm]{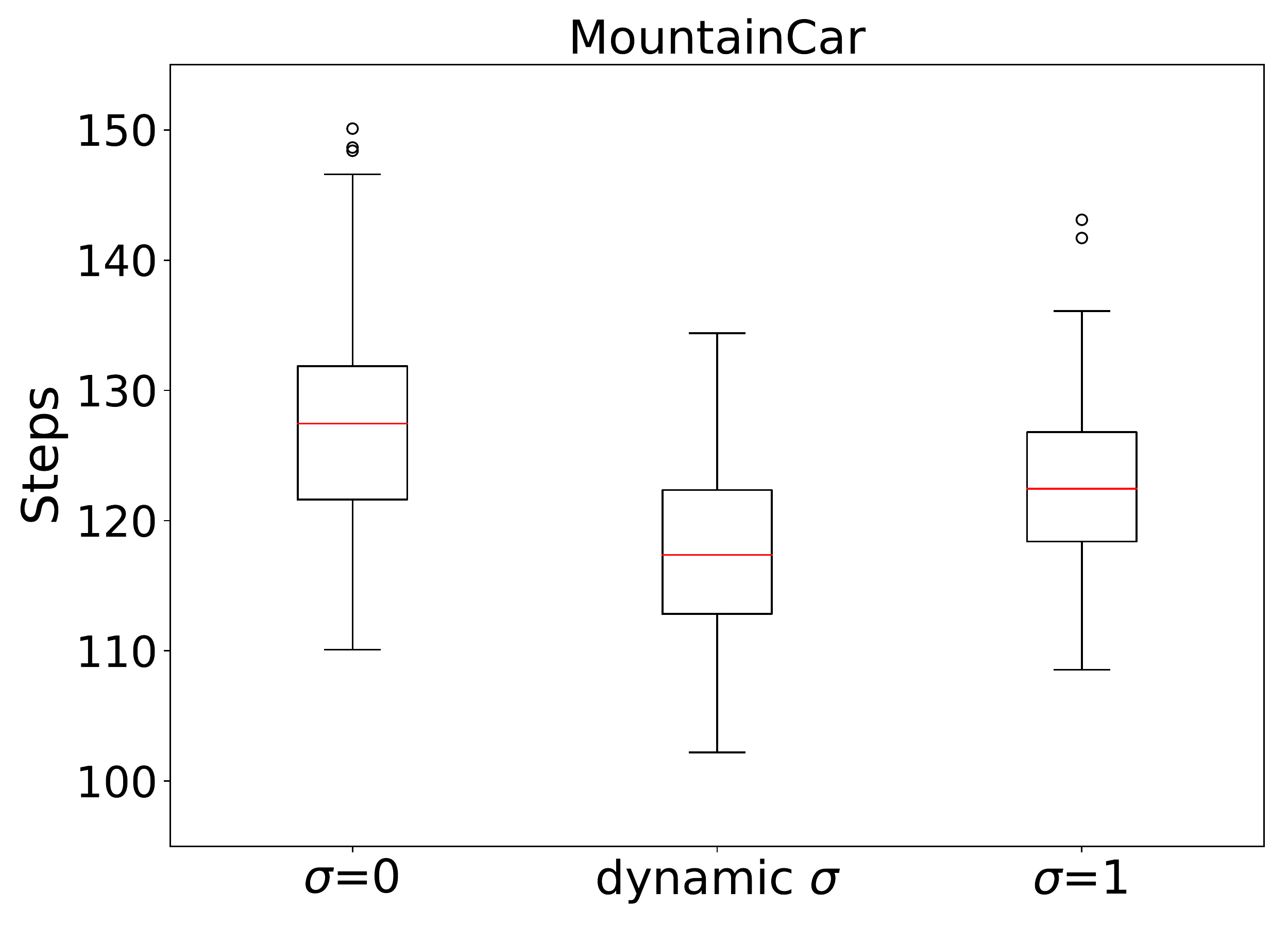}}
    \caption
    {
        Comparison of variance. Figure shows the results of $p=1024$.
    }
\end{figure}

\begin{table}[H]
    \centering
    
    \begin{tabular}{|l|l|c|c|c|}
        \hline
        \multicolumn{2}{|c|}{Case}                & $p=512$             & $p=1024$            & $p=2048$            \\ \hline
        \multirow{3}{*}{$\alpha=0.001$} & I   & \textbf{69.8\%} & 20.4\%& \textbf{55.1\%} \\ \cline{2-5} 
        & II & 14.1\%          & 36.7\%          & 20.4\%          \\ \cline{2-5} 
        &  III         & 16.1\%          & \textbf{42.9\%}          & 24.5\%          \\ \hline 
        \multirow{3}{*}{$\alpha=0.002$} &  I    & \textbf{53.1\%} & 6.1\%  & \textbf{49.0\%} \\ \cline{2-5} 
        & II  & 38.8\%          & 16.3\%          & 18.4\%          \\ \cline{2-5} 
        & III          & 8.1\%           & \textbf{77.6\%}          & 32.6\%          \\ \hline
    \end{tabular}
    \label{tab:mc}
    \caption{Percentage under various parameters.}
\end{table}

Table 1 shows that GQ$(\sigma,\lambda)$ performs better than both $\sigma=0$ and $\sigma=1$ with a more significant number than case II and III.
Table 2 implies the GQ$(\sigma,\lambda)$ reaches the best performance and the advantage is more significant for $p=512 $ and $2048$.

This experiment further illustrates
the trade-off between full-sampling and pure-expectation in RL for control tasks.

\section{Conclusions}

In this paper, we study tabular Q$(\sigma,\lambda)$ with function approximation.
We analyze divergence in Q$(\sigma,\lambda)$ with semi-gradient method.
To address the instability of semi-gradient Q$(\sigma,\lambda)$ algorithm, we propose GQ$(\sigma,\lambda)$ algorithm.
Our theoretical results have given a guarantee of convergence for the combination of full-sampling algorithm and pure-expectation algorithm, which is of great significance in reinforcement learning algorithms with function approximation.
Finally, we conduct extensive experiments on some standard domains to show that GQ$(\sigma,\lambda)$ with an value $\sigma\in(0,1)$ that results in a mixture of the full-sampling with pure-expectation methods, performs better than either extreme $\sigma=0$ or $\sigma=1$.

\bibliographystyle{aaai}
\bibliography{reference}

\clearpage

\onecolumn
\appendix
\section{Appendix A}
\subsection{Proof of Eq.(\ref{Eq:linear_eq})}


Let
$
{P}^{\pi}_{\lambda}=(1-\lambda)\sum_{\ell=0}^{\infty}\gamma^{\ell}\lambda^{\ell}({P}^{\pi})^{\ell+1},\hspace{0.3cm}
\mathcal{R}_{\lambda}^{\pi}=\sum_{\ell=0}^{\infty}\gamma^{\ell}\lambda^{\ell}({P}^{\pi})^{\ell}\mathcal{R}^{\pi}.
$

\textbf{Step1}: If ${A}_{k}={\phi_{k}}\{\sum_{t=k}^{\infty}(\gamma\lambda)^{t-k}(\phi_{t}-\gamma\phi_{t+1})^{T}\}$, then
\[\mathbb{E}[{A}_{k}]={\Phi}^{T}{\Xi}(I-\gamma\lambda{P}^{\mu})^{-1}({I}-\gamma{P}^{\mu}){\Phi}.\]
\begin{eqnarray*}
&&\sum_{t=k}^{\infty}(\gamma\lambda)^{t-k}(\phi_{t}-\gamma\phi_{t+1})^{\top}\\
&=&\lim_{n\rightarrow\infty}[\sum_{t=k}^{n}(\gamma\lambda)^{t-k}(\phi_{t}-\gamma\phi_{t+1})^{\top}]\\
&=&\lim_{n\rightarrow\infty}[{\phi}^{\top}_{k}-\gamma{\phi}^{\top}_{k+1}+\sum_{t=k+1}^{n}(\gamma\lambda)^{t-k}{\phi}^{\top}_{t}-\gamma\sum_{t=k+1}^{n}(\gamma\lambda)^{t-k}{\phi}^{\top}_{t+1}]\\
&=&\lim_{n\rightarrow\infty}[{\phi}^{\top}_{k}+\sum_{t=k}^{n-1}(\gamma\lambda)^{t+1-k}{\phi}^{\top}_{t+1}-\gamma\sum_{t=k}^{n}(\gamma\lambda)^{t-k}{\phi}^{\top}_{t+1}]\\
&=&\lim_{n\rightarrow\infty}[{\phi}^{\top}_{k}+\gamma\lambda\sum_{t=k}^{n-1}(\gamma\lambda)^{t-k}{\phi}^{\top}_{t+1}-\gamma\sum_{t=k}^{n}(\gamma\lambda)^{t-k}{\phi}^{\top}_{t+1}]\\
&=&\lim_{n\rightarrow\infty}[{\phi}^{\top}_{k}-\gamma(1-\lambda)\sum_{t=k}^{n-1}(\gamma\lambda)^{t-k}{\phi}^{\top}_{t+1}-\gamma^{n-k+1}\lambda^{n-k}{\phi}^{\top}_{n+1}]\\
&=&{\phi}^{\top}_{k}-\gamma(1-\lambda)\sum_{t=k}^{\infty}(\gamma\lambda)^{t-k}{\phi}^{\top}_{t+1}.
\end{eqnarray*}

For a stable behavior policy $\mu$, we have $\mathbb{E}[A_{k}]=\mathbb{E}[A_{0}]$.
Thus,
\begin{flalign}
\mathbb{E}[{A}_{k}]
\nonumber
&=\mathbb{E}[{\phi}_{k}{\phi}^{\top}_{k}]-\gamma\mathbb{E}[(1-\lambda)\sum_{t=0}^{\infty}(\gamma\lambda)^{t}\phi^{\top}_{t+1}]\\
\label{A1}
&={\Phi}^{\top}{\Xi}{\Phi}-\gamma{\Phi}^{\top}{\Xi}{{P}}^{\mu}_{\lambda}{\Phi}\\
\nonumber
&={\Phi}^{\top}{\Xi}\{{I}-\gamma(1-\lambda)\sum_{\ell=0}^{\infty}\gamma^{\ell}\lambda^{\ell}({P}^{\mu})^{\ell+1}\}{\Phi}\\
\nonumber
&={\Phi}^{\top}{\Xi}\{I-\gamma(1-\lambda)(I-\gamma\lambda{P}^{\mu})^{-1}{P}^{\mu}\}{\Phi}\\
\nonumber
&={\Phi}^{\top}{\Xi}(I-\gamma\lambda{P}^{\mu})^{-1}({I}-\gamma{P}^{\mu}){\Phi}.
\end{flalign}
By the identity ${P}^{\mu}_{\lambda}=(1-\lambda)\sum_{\ell=0}^{\infty}\gamma^{\ell}\lambda^{\ell}({P}^{\mu})^{\ell+1}$, thus (\ref{A1}) holds.

\textbf{Step2}: If ${A}_{k}={\phi}_{k}\{\sum_{t=k}^{\infty}(\gamma\lambda)^{t-k}({\phi}_{t}-\gamma\mathbb{E}_{\pi}[\phi(S_{t+1},\cdot)])^{\top}\}$, then
\[\mathbb{E}[{A}_{k}]={\Phi}^{\top}{\Xi}(I-\gamma\lambda{P}^{\mu})^{-1}({I}-\gamma{P}^{\pi}){\Phi}.\]

Under Assumption 1and it is similar to the proof in step 1, then we have,
\begin{eqnarray*}
\sum_{t=k}^{\infty}(\gamma\lambda)^{t-k}({\phi}_{t}-\gamma\mathbb{E}_{\pi}[\phi(S_{t+1},\cdot)])^{\top}&=&{\phi}^{T}_{k}-\gamma(1-\lambda)\sum_{t=k}^{\infty}(\gamma\lambda)^{t-k}\mathbb{E}_{\pi}[\phi^{\top}(S_{t+1},\cdot)]\\
&=&{\phi}^{\top}_{k}-\gamma(1-\lambda)\sum_{t=0}^{\infty}(\gamma\lambda)^{t}\mathbb{E}_{\pi}[\phi^{\top}(S_{t+1},\cdot)]
\end{eqnarray*}
Let $\mathcal{E}_{t}=\cup_{i=0}^{t}\{(A_{i},S_{i},R_{i})\},\mathcal{F}_{t}=\mathcal{E}_{t}\cup\{S_{t+1}\}$, then
\begin{eqnarray*}
\mathbb{E}\Big[\sum_{t=0}^{\infty}(\gamma\lambda)^{t}\mathbb{E}_{\pi}[\phi^{\top}(S_{t+1},\cdot)]\Big]
&=&\sum_{t=0}^{\infty}(\gamma\lambda)^{t}\mathbb{E}_{\mathcal{F}_{t}}\Big[\mathbb{E}_{\pi}[\phi^{T}(S_{t+1},\cdot)]\Big]\\
&=&\sum_{t=0}^{\infty}(\gamma\lambda)^{t}\mathbb{E}_{\mathcal{E}_{t}}\Big\{\mathbb{E}_{S_{t+1}\sim \mathcal{P}(\cdot |{A_{t}},{S_{t}})}\Big[\mathbb{E}_{\pi}[\phi^{\top}(S_{t+1},\cdot)]\Big]\Big\}\\
&=&\sum_{t=0}^{\infty}(\gamma\lambda)^{t}\mathbb{E}_{\mathcal{E}_{t}}\Big\{\sum_{a\in\mathcal{A}}\sum_{s\in\mathcal{S}}\mathcal{P}(s|S_{t},A_{t})\pi(s,a)\phi(s,a)\Big\}\\
&=&\sum_{t=0}^{\infty}(\gamma\lambda)^{t}\mathbb{E}_{\mathcal{E}_{t}}\{{{P}^{\pi}}\phi_{t}\}\\
&=&\sum_{t=0}^{\infty}(\gamma\lambda)^{t}\mathbb{E}_{\mathcal{E}_{t-1}}\Big\{\sum_{a\in\mathcal{A}}\sum_{s\in\mathcal{S}}\mathcal{P}(s|S_{t-1},A_{t-1})\mu(s,a){P}^{\pi}\phi_{t}\Big\}\\
&=&\sum_{t=0}^{\infty}(\gamma\lambda)^{t}\mathbb{E}_{\mathcal{E}_{t-1}}\Big\{{P}^{\mu}{P}^{\pi}\phi_{t}\Big\}\\
&=&\sum_{t=0}^{\infty}(\gamma\lambda)^{t}\mathbb{E}_{\mathcal{E}_{t-2}}\Big\{({P}^{\mu})^{2}{P}^{\pi}\phi_{t-2}\Big\}\\
&=&\sum_{t=0}^{\infty}(\gamma\lambda)^{t}\mathbb{E}\Big\{({P}^{\mu})^{t}{P}^{\pi}\phi_{0}\Big\}.
\end{eqnarray*}
\begin{eqnarray*}
\mathbb{E}[A_{k}]&=&\mathbb{E}\Big[{\phi}_{k}\{\sum_{t=k}^{\infty}(\gamma\lambda)^{t-k}({\phi}_{t}-\gamma\mathbb{E}_{\pi}[\phi(S_{t+1},\cdot)])^{\top}\}\Big]\\
&=&\mathbb{E}\Big[{\phi}_{t}\Big\{{\phi}^{\top}_{t}-\gamma(1-\lambda)\sum_{t=0}^{\infty}(\gamma\lambda)^{t}\mathbb{E}_{\pi}[\phi^{\top}(S_{t+1},\cdot)]\Big\}\Big]\\
&=&\mathbb{E}[\phi_{t}\phi^{\top}_{t}]-\gamma(1-\lambda)\sum_{t=0}^{\infty}(\gamma\lambda)^{t}\mathbb{E}\Big\{\phi_{t}^{\top}({P}^{\mu})^{t}{P}^{\pi}\phi_{t}\Big\}\\
&=&\mathbb{E}\Big[\phi_{0}\Big(\sum_{t=0}^{\infty}(\gamma\lambda)^{t}({P}^{\mu})^{t}\Big)(I-\gamma{P}^{\pi})\phi^{\top}_{0}\Big]\\
&=&{\Phi}^{\top}{\Xi}(I-\gamma\lambda{P}^{\mu})^{-1}({I}-\gamma{P}^{\pi}){\Phi}.
\end{eqnarray*}

\textbf{Step3}: Combining the step1 and step2, we have:
\[
A_{\sigma}={\Phi}^{T}\Xi(I-\gamma\lambda{\mathcal{P}}^{\mu})^{-1}((1-\sigma)\gamma{P}^{\pi}+\sigma\gamma{P}^{\mu}-{I}){\Phi}.
\]
The proof of $\mathbb{E}[\hat{b}_{k}]$ is similar to above steps and we omit it.
\[
b_{\sigma}=\Phi\Xi (I-\gamma\lambda{\mathcal{P}}^{\mu})^{-1}(\sigma\mathcal{R}^{\mu}+(1-\sigma)\mathcal{R}^{\pi}).
\]

\textbf{Step4}: 
Taking expectation of Eq.(\ref{Eq:semi-gradient}), by Step 1, Step 2 and Step 3, we have  Eq.(\ref{Eq:linear_eq}).

\clearpage
\section{Appendix B: Proof of Proposition \ref{prop1}}

\textbf{Proposition \ref{prop1}}\emph{
	Let $e_{t}$ be the eligibility traces vector that is generated as $e_{k}=\lambda\gamma e_{k-1}+\phi_{k}$, let
	\begin{flalign}
\nonumber
\Delta_{\sigma,k}&=\gamma\{\sigma\phi_{k+1}+(1-\sigma)\mathbb{E}_{\pi}\phi(S_{k+1},\cdot)\}-\phi_{k},\\
\nonumber
v_{\sigma}(\theta_{k})&=(1-\sigma)\{\mathbb{E}_{\pi}\phi(S_{k+1},\cdot)-\lambda\phi_{k+1}\}e^{\top}_{k}+\sigma(1-\lambda)\phi_{k+1}e^{\top}_{k},
\end{flalign}
then we have 
\begin{flalign}
\nonumber
\theta_{k+1}&=\theta_{k}-\alpha_{k}\frac{1}{2}\nabla_{\theta} \emph{MSPBE}(\theta,\lambda)|_{\theta=\theta_{k}}\\
\nonumber
&=\theta_{k}-\alpha_{k}\mathbb{E}[\Delta_{\sigma,k}e_{k}^{\top}]\omega(\theta_{k})\\
\nonumber
&=\theta_{k}+\alpha_{k}\{\mathbb{E}[\delta_{k}e_{k}]-\gamma\mathbb{E}[v_{\sigma}(\theta_{k})]\omega(\theta_{k})\}.
\end{flalign}
}

\begin{proof}
	Let us calculate MSPBE($\theta,\lambda$) directly,
\begin{flalign}
\nonumber
&-\frac{1}{2}\nabla_{\theta} \text{MSPBE}(\theta,\lambda)|_{\theta=\theta_{k}}\\
\nonumber
=&-\frac{1}{2}\nabla_{\theta}\Big(\mathbb{E}[\delta_{k}e_{k}]^{\top}\mathbb{E}[\phi_{k}\phi^{\top}_{k}]^{-1}\mathbb{E}[\delta_{k}e_{k}]\Big)\\
\nonumber
=&-(\nabla_{\theta}\mathbb{E}[\delta_{k}e_{k}]^{\top})\mathbb{E}[\phi_{k}\phi^{\top}_{k}]^{-1}\mathbb{E}[\delta_{k}e_{k}]\\
\nonumber
=&-\mathbb{E}\Big[\underbrace{\Big(\gamma(\sigma\phi_{k+1}+(1-\sigma)\mathbb{E}_{\pi}\phi(S_{k+1},\cdot))-\phi_{k}\Big)}_{\Delta_{\sigma,k}}e^{\top}_{k}\Big]\omega(\theta_k)\\
\nonumber
=&-\mathbb{E}\Big[\gamma(\sigma\phi_{k+1}+(1-\sigma)\mathbb{E}_{\pi}\phi(S_{k+1},\cdot))e^{\top}_{k}-\phi_{k}e^{\top}_{k}\Big]\omega(\theta_k)\\
\nonumber
=&\mathbb{E}\Big[\phi_{k}\phi^{T}_{k}+\phi_{k+1}\gamma\lambda e_{k}^{\top}-\gamma(\sigma\phi_{k+1}+(1-\sigma)\mathbb{E}_{\pi}\phi(S_{k+1},\cdot))e^{\top}_{k}\Big]\omega(\theta_k)
\\
\label{prop1-1}
=&\mathbb{E}[\delta_{k}e_{k}]-\gamma\mathbb{E}\Big[\sigma(1-\lambda)\phi_{k+1}e^{\top}_{k}+(1-\sigma)\Big(\mathbb{E}_{\pi}\phi(S_{k+1},\cdot)-\lambda\phi_{k+1}\Big)e^{\top}_{k}\Big]\omega(\theta).
\end{flalign}
Taking Eq.(\ref{prop1-1}) into Eq.(\ref{theata_update}), then we have Eq.(\ref{theta_uptate_2}).
\end{proof}

\section{Appendix C}

In fact, $\Delta_{\sigma,k}=(1-\sigma)[\gamma\mathbb{E}_{\pi}\phi(S_{k+1},\cdot)-\phi(S_{k},A_{t})]+\sigma[\gamma\phi(S_{k+1},A_{k+1})-\phi(S_{k},A_{k})]$
\begin{eqnarray*}
\mathbb{E}[\hat{A_{k}}]&=&\mathbb{E}[\Delta_{\sigma,k}e_{k}]\\
&=&\mathbb{E}[\Delta_{\sigma,k}(\sum^{k}_{i=0}(\lambda\gamma)^{k-i}\phi(S_{i},A_{i}))]\\
&{=}&\mathbb{E}[\phi(S_{k},A_{k})\Delta_{\sigma,k}+\lambda\gamma\phi(S_{k-1},A_{k-1})\Delta_{\sigma,k}+\sum^{\infty}_{t=k+1}(\lambda\gamma)^{t-k+1}\phi(S_{k},A_{k})\Delta_{\sigma,t+1}]\\
&=&\mathbb{E}[\phi(S_{k},a_{k})\Delta_{\sigma,k}+\lambda\gamma\phi(S_{k},A_{k})\Delta_{\sigma,k+1}+\sum^{\infty}_{t=k+1}(\lambda\gamma)^{t-k+1}\phi(S_{k},A_{k})\Delta_{\sigma,t+1}]\\
&=&\mathbb{E}[\phi(S_{k},A_{k})\Delta_{\sigma,k}+\sum^{\infty}_{t=k}(\lambda\gamma)^{t-k+1}\phi(S_{k},A_{k})\Delta_{\sigma,t+1}]\\
&=&\mathbb{E}[\phi(S_{k},A_{k})\Delta_{\sigma,k}+\sum^{\infty}_{t=k+1}(\lambda\gamma)^{t-k}\phi(S_{k},A_{k})\Delta_{\sigma,t}]\\
&=&\mathbb{E}[\sum^{\infty}_{t=k}(\lambda\gamma)^{t-k}\phi(S_{k},A_{k})\Delta_{\sigma,t}]\\
&=&A_{\sigma}
\end{eqnarray*}

\section{Appendix D: Lemma \ref{Borkar-two--timescale}}

\begin{lemma}[\cite{borkar1997stochastic}]
	\label{Borkar-two--timescale}
	For the stochastic recursion of $x_{n},y_{n}$ given by
	\begin{flalign}
	\label{Borkar97-lemma-x}
	x_{n+1}&=x_{n}+a_{n}[g(x_{n},y_{n})+M^{(1)}_{n+1}],\\
	\label{Borkar97-lemma-y}
	y_{n+1}&=y_{n}+b_{n}[h(x_{n},y_{n})+M^{(2)}_{n+1}],n\in\mathbb{N}
	\end{flalign}
	if the following assumptions are satisfied:
	\begin{itemize}
		\item (A1)Step-sizes $\{a_{n}\},\{b_{n}\}$ are positive, satisfying
		\[
		\sum_{n} a_{n}=\sum_{n} b_{n}=\infty, \sum_{n} a^{2}_{n}+b^{2}_{n}<\infty,\dfrac{b_{n}}{a_{n}}\rightarrow 0~~ \emph{as} ~~n\rightarrow \infty.  
		\]
		\item(A2) The map $g:\mathbb{R}^{d+k} \rightarrow \mathbb{R}^{d},h:\mathbb{R}^{d+k} \rightarrow \mathbb{R}^{k}$ are Lipschitz.
		\item(A3) The sequence $\{M^{(1)}_{n+1}\}_{n\in\mathbb{N}},\{M^{(2)}_{n+1}\}_{n\in\mathbb{N}}$ are martingale difference sequences w.r.t. the increasing $\sigma$-fields
		$
		\mathcal{F}_{n} \overset{\text{def}}= \sigma(x_{m},y_{m},M^{(1)}_{m},M^{(2)}_{m}, m \leq n), n\in\mathbb{N},
		$
		satisfying \[\mathbb{E}[M^{(i)}_{n+1}|\mathcal{F}_{n}]=0,i=1,2, n\in\mathbb{N}.\]
		Furthermore, $\{M^{(i)}_{n+1}\}_{n\in\mathbb{ N}}, i=1,2$, are square-integrable with
		\[
		\mathbb{E}[\|M^{(i)}_{n+1}\|^{2}|\mathcal{F}_{n}]\leq K(1+\|x_{n}\|^{2}+\|y_{n}\|^{2}),
		\]
		for some constant $K > 0$.
		\item(A4)For each $x\in\mathbb{R}^{d}$,the o.d.e.
		\[
		\dot y(t) = h(x, y(t))
		\]
		has a  global asymptotically stable equilibrium $\Omega(x)$ such that:$\Omega(x):\mathbb{R}^{d}\rightarrow\mathbb{R}^{k}$ is Lipschitz.
		\item(A5)The o.d.e.\[\dot x(t) =g(x(t),\Omega(x(t)))\] has a  global asymptotically stable equilibrium $x^{*}$.
	\end{itemize}
	Then, the iterates (\ref{Borkar97-lemma-x}), (\ref{Borkar97-lemma-y}) converge to $(x^{*},\Omega(x^{*}))$ a.s. 
	on the set $Q \overset{\emph{def}}=\{\sup_{n} x_{n}<\infty, \sup_{n} y_{n}<\infty\}.$
\end{lemma}

\section{Appendix E}

\subsection{Baird Star}
\label{app-ex-Baird Star}

Baird's Star a well known example for divergence in off-policy TD learning for all step-sizes.
The Baird’s Star considers the episodic MDP with 
seven-state $\mathcal{S}=\{\mathtt{s}_1,\cdots,\mathtt{s}_7\}$ and two-action 
$\mathcal{A}$= \{$\mathtt{dashed}$ action, $\mathtt{solid}$ action\}.
The $\mathtt{dashed}$ action takes the system
to one of the first six states with equal probability, 
whereas the $\mathtt{solid}$ action takes the
system to the state $\mathtt{s}_7$. 
The policy $\mu$ and $\pi$ select the $\mathtt{dashed}$ and $\mathtt{solid}$ actions
with probabilities \[\mu(\mathtt{dashed}|\cdot)=\frac{6}{7}, \mu(\mathtt{solid}|\cdot)=\frac{1}{7}, \pi(\mathtt{solid}|\cdot)=1,\]
which implies that the target policy always takes the $\mathtt{solid}$ action.
The reward is zero on all transitions. The discount rate $\gamma= 0.99.$
Features are chosen as \[\phi(\mathtt{s}_i) = 2\epsilon_i + (0, 0,0, 0, 0 ,0 ,0, 1)^{T},\] where $i$-th component is $\epsilon_i$ is 1, others are all 0.

We used $\theta_{0} = (1, 1, 1 ,1, 1 ,1, 10, 1)^{T}$ as initial parameter vector for the methods that allow specifying a start estimate, TD-learning is known to diverge for this initialization of the parameter-vector \cite{dann2014policy,sutton2018reinforcement}.

\subsection{Boyan Chain}

\begin{figure}[htbp]
	\centering
	{\includegraphics[scale=0.6]{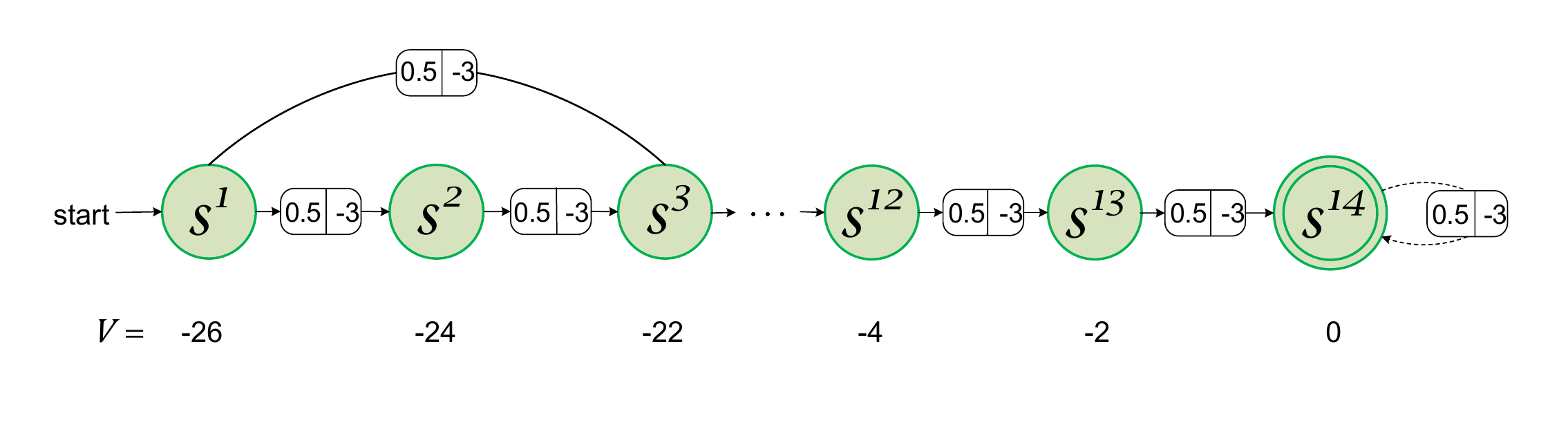}}	
	\caption
	{
		The dynamics of Boyan Chain 
	}
\end{figure}
The second benchmark MDP is the classic chain example from \cite{boyan2002technical} which considers a chain of $14$ states $\mathcal{S} = \{\mathtt{s}_1, \cdots, \mathtt{s}_{14}\}$ and one action. 
Each transition from state $\mathtt{s}_i$ results in state $\mathtt{s}_{i+1}$ or $\mathtt{s}_{i+2}$ with equal probability and a reward of $-3$. 
If the agent is in the second last state $\mathtt{s}_{13}$, it always proceeds to the last state with reward  $-2$ and subsequently stays in this state forever with zero reward.

The true value function, which is linearly decreasing from $\mathtt{s}_1$ to $\mathtt{s}_{14}$, can be represented perfectly.
\begin{flalign}
\nonumber
P_{\mathtt{14}\times\mathtt{14}} = \begin{pmatrix}
0 & \frac{1}{2} & \frac{1}{2} & 0&0 & \cdots&0 & 0  \\
0 & 0 & \frac{1}{2} & \frac{1}{2}&0 & \cdots&0 & 0  \\
0 & 0 & 0 & \frac{1}{2}&\frac{1}{2} & \cdots&0 & 0  \\
 \vdots&&&&&&&  \vdots\\
0 & 0 &0 & 0 & 0& \cdots&\frac{1}{2} & \frac{1}{2}   \\
0 & 0 & 0 & 0 & 0&\cdots&0 & 1  \\
0 & 0 & 0 & 0 & 0&\cdots&0 & 1   
\end{pmatrix},~~~
R_{\mathtt{14}\times\mathtt{1}} = \begin{pmatrix}
-3  \\
-3  \\
-3   \\
\vdots\\
-3 \\
-2 \\
0  
\end{pmatrix} ~
\end{flalign}
By Bellman equation, we have 
\[
v=(I-\gamma P)^{-1}R~~\rightarrow
(-26 ,-24  -22,\cdots,   -4, -2 ,0 ) ,~~\text{as} ~\gamma\rightarrow~1.
\]

In this paper, we chose a discount factor of $\gamma= 0.99$ and four-dimensional feature description with triangular-shaped basis functions covering the state space (Figure 7). 

\begin{figure}[h]
	\label{fig:app-boyan-chain-feature}
	\vskip 0.2in
	\begin{center}
		\centerline{\includegraphics[width=7cm]{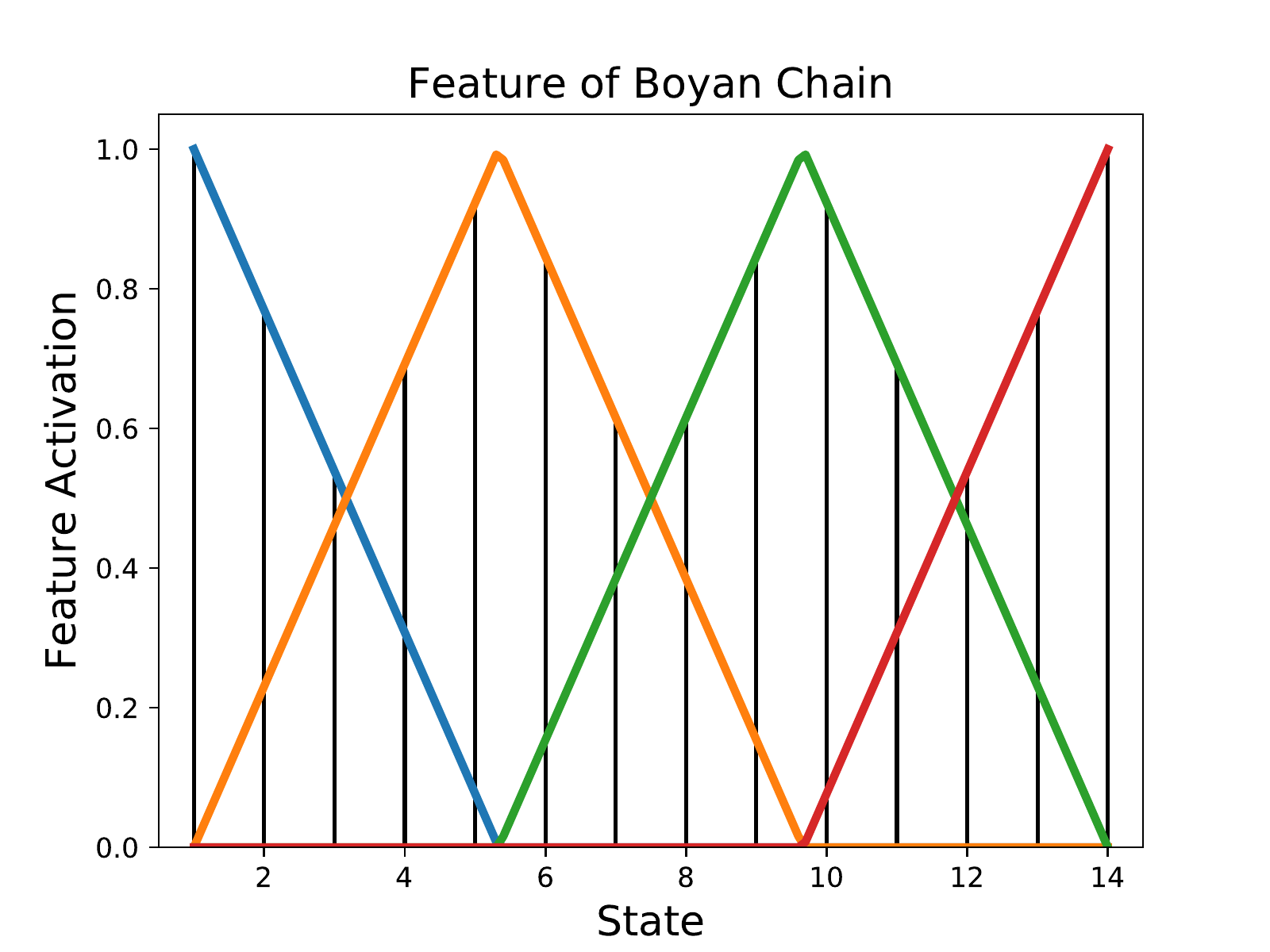}}
		\caption{Feature Activation for the Boyan chain benchmark. The state space is densely covered with triangle-shaped basis functions. The figure here refers to \cite{dann2014policy}.}
		\label{icml-historical}
	\end{center}
	\vskip -0.2in
\end{figure}

\end{document}